\documentclass[runningheads]{llncs}

\PassOptionsToPackage{table}{xcolor}
\usepackage[final]{eccv}

\usepackage{eccvabbrv}
\usepackage{graphicx}
\usepackage{booktabs}
\usepackage{amsmath,amssymb}
\usepackage{multirow}
\usepackage{microtype}
\usepackage[accsupp]{axessibility}
\usepackage[hidelinks]{hyperref}
\hypersetup{pdftitle={FILT3R: Latent State Adaptive Kalman Filter for Streaming 3D Reconstruction}, pdfauthor={Seonghyun Jin and Jong Chul Ye}}
\usepackage{orcidlink}
\usepackage{xcolor}
\usepackage{pgfplots}
\pgfplotsset{compat=1.18}
\usepackage{capt-of}
\usepackage{tikz}
\usepackage[section]{placeins}
\usepackage{algorithm}
\usepackage{algpseudocode}
\usetikzlibrary{arrows.meta, calc, shapes.geometric}
\newcommand{\prior}{\(\dagger\)}
\makeatletter

\usetikzlibrary{arrows.meta, positioning, fit, calc}

\makeatother

\definecolor{oursbg}{RGB}{232,242,255}

\newcommand{\oursrow}{\rowcolor{oursbg}}
\newcommand{\oursname}{\textbf{FILT3R (ours)}}

\definecolor{lightred}{RGB}{255, 235, 235}

\newcommand{\appendixindexentry}[2]{%
\noindent\begin{tabular*}{\linewidth}{@{}p{0.88\linewidth}@{\extracolsep{\fill}}r@{}}
\hyperref[#1]{#2} & \pageref{#1}
\end{tabular*}\par}
\newcommand{\appendixsubindexentry}[2]{%
\noindent\hspace*{1.5em}\begin{tabular*}{0.95\linewidth}{@{}p{0.84\linewidth}@{\extracolsep{\fill}}r@{}}
\hyperref[#1]{#2} & \pageref{#1}
\end{tabular*}\par}

\begin{document}

\title{FILT3R: Latent State Adaptive Kalman Filter \texorpdfstring{\\}{ } for Streaming 3D Reconstruction}
\titlerunning{FILT3R}

\author{Seonghyun Jin\inst{1}\orcidlink{0009-0007-8952-5178} \and
Jong Chul Ye\inst{1}\orcidlink{0000-0001-9763-9609}\thanks{Corresponding author.}}

\authorrunning{S. Jin and JC. Ye}

\institute{Kim Jaechul Graduate School of AI, KAIST, Seoul, South Korea\\
\email{jinotter3@kaist.ac.kr, jong.ye@kaist.ac.kr}\\
\url{https://github.com/jinotter3/FILT3R}}

\maketitle

\begin{abstract}
Streaming 3D reconstruction maintains a persistent latent state that is updated online from incoming frames,
enabling constant-memory inference.
A key failure mode is the state update rule: aggressive overwrites forget useful history,
while conservative updates fail to track new evidence,
and both behaviors become unstable beyond the training horizon.
To address this challenge, we propose \textbf{FILT3R}, a training-free latent filtering layer that casts recurrent state updates
as stochastic state estimation in token space.
FILT3R maintains a per-token variance and computes a Kalman-style gain that adaptively balances
memory retention against new observations.
Process noise -- governing how much the latent state is expected to change between frames -- is
estimated online from EMA-normalized temporal drift of candidate tokens.
Using extensive experiments, we demonstrate that FILT3R yields an interpretable, plug-in update rule that generalizes common overwrite
and gating policies as special cases.
Specifically, we show that
gains shrink in stable regimes as uncertainty contracts with accumulated evidence,
and rise when genuine scene change increases process uncertainty,
improving long-horizon stability for depth, pose, and 3D reconstruction, compared to the existing methods.
\end{abstract}

\section{Introduction}

\begin{figure}[t]
\centering
\includegraphics[width=\linewidth]{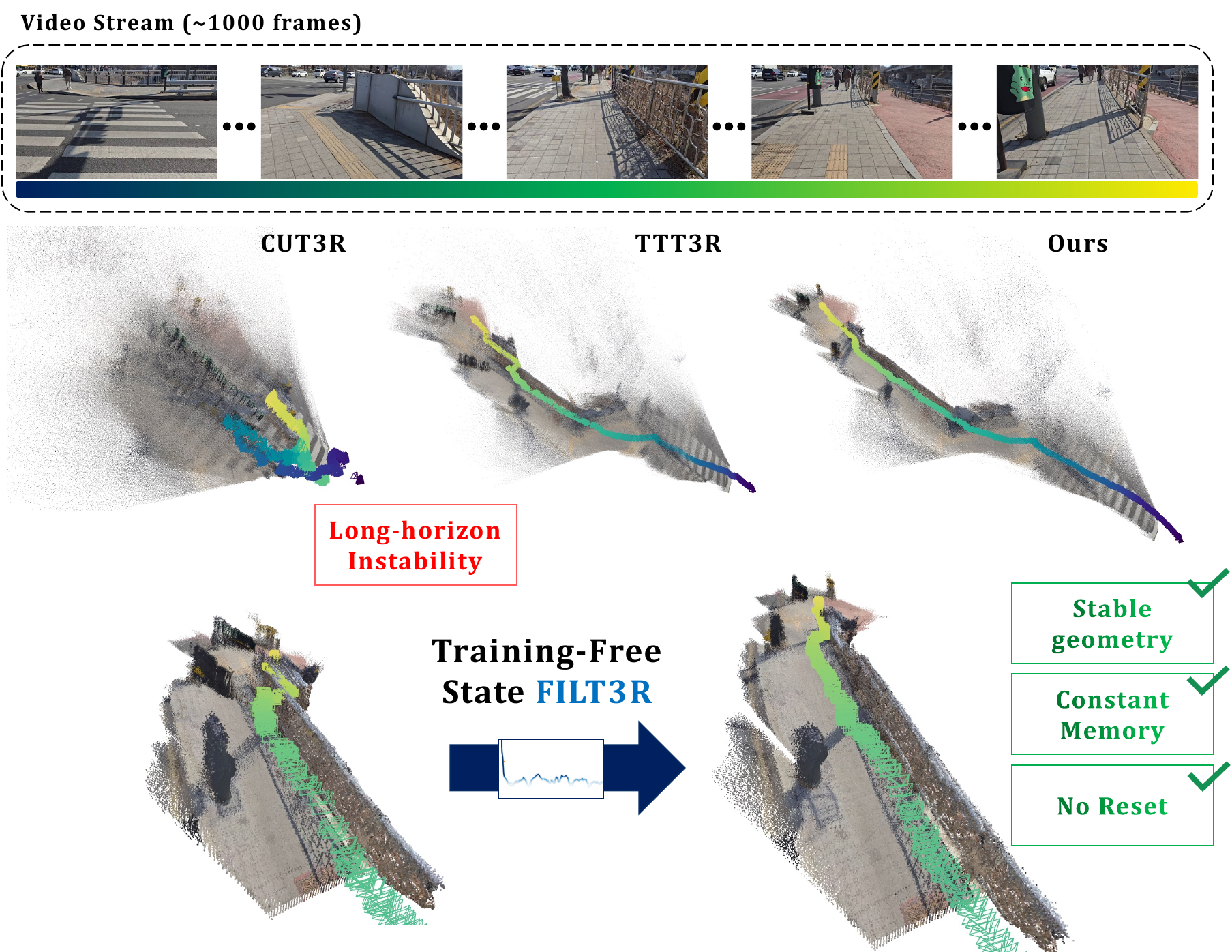}
\caption{\textbf{FILT3R enables stable long-horizon streaming 3D reconstruction without resets.}
Top: dense input frames from a long video stream.
Middle: final reconstructions from the same frozen streaming backbone under different latent-state update rules.
Uniform overwrite (CUT3R) and heuristic gating (TTT3R) accumulate drift and/or forget previously integrated evidence, leading to long-horizon instability and degraded geometry.
Bottom: a zoom-in around frames beyond 600 highlights that FILT3R preserves stable geometry over long rollouts.
Trajectory colors indicate time.}
\label{fig:main}
\end{figure}

Streaming 3D reconstruction aims to predict geometry and camera motion from an image stream
while using constant memory and low latency.
Recent recurrent architectures achieve this by maintaining a compact set of persistent latent tokens, which distill the scene's context and evolve dynamically as new frames arrive~\cite{wang2025cut3r}.
In practice, performance often degrades when the sequence length exceeds the training horizon:
small update biases accumulate into drift, and the latent state gradually loses information
that is no longer supported by the most recent frame.
Since a new candidate latent state should be estimated from the current frame and the
previous state,
the system must then decide how much to trust the candidate (fast adaptation) versus the
previous state (long-term consistency).

Existing streaming update rules are mostly
 manually engineered to manage this trade-off.
 Uniform overwrites (\eg, CUT3R~\cite{wang2025cut3r}) always trust the new candidate, which can lead to
catastrophic forgetting.
While training-free gates based on attention statistics (\eg, TTT3R~\cite{chen2026ttt3r}) improve stability, they remain largely heuristic; they scale updates without an explicit representation of uncertainty or a principled mechanism for adaptation as confidence increases.
Mathematically, existing streaming updates can be written as
$\mathbf{s}_t=(1-\boldsymbol\beta_t)\odot\mathbf{s}_{t-1}+\boldsymbol\beta_t\odot\tilde{\mathbf{s}}_t$,
where $\boldsymbol\beta_t\in[0,1]^N$ determines how quickly past evidence is forgotten.
Uniform overwrite ($\boldsymbol\beta_t=\mathbf{1}$) yields a one-step memory,
while fixed-gain gates impose a constant forgetting half-life regardless of stream stability.

Inspired by this formulation,
we argue that the recurrent state in streaming reconstruction is naturally a \emph{belief state},
and that the candidate state produced by the decoder is a noisy \emph{measurement} of that belief.
This viewpoint leads directly to our novel adaptive filtering formulation, which we call \textbf{FILT3R}.
In particular,
FILT3R formulates the state update as a
Kalman-style interpolation whose gain is governed by two uncertainties:
(i)~\emph{process noise}---how much the underlying scene latent is expected to change between frames---and
(ii)~\emph{measurement noise}---how reliable the decoder's current prediction is.
Specifically, FILT3R adopts an adaptive Kalman-style filtering (AKF) framework~\cite{mehra1970akf} with a fixed measurement model. Adaptivity is centered on the process model, allowing Kalman-style gains to attenuate as tokens gain confidence during stable periods -- thereby extending the effective memory horizon -- and increase in response to process noise that indicates a genuine scene change.
More specifically, FILT3R uses a \emph{fixed} scalar measurement noise $r$
and estimates only the \emph{process noise} adaptively from temporal drift in the latent space.
This formulation not only reduces computational complexity but also circumvents a critical failure mode, as validated by our empirical findings.

Our contribution can be summarized as follows:
\begin{itemize}
\item[(i)]
We recast streaming 3D reconstruction as stochastic state estimation in latent token space, and
 introduce \textbf{FILT3R}, an adaptive latent filtering layer for streaming 3D reconstruction.
 \item[(ii)]
 FILT3R is a training-free, lightweight, and interpretable plug-and-play filtering approach with
adaptive process noise estimated from internal model signals and fixed measurement noise,
following the classical AKF paradigm.
\item[(iii)] We provide extensive evaluation across short- and long-horizon streaming
depth, pose, and reconstruction settings,
and analyze how uncertainty-aware gains improve length generalization. The results
confirm that FILT3R consistently outperforms prior streaming baselines and improves long-horizon stability across metrics.
\end{itemize}

\section{Related Work}

\paragraph{Streaming and recurrent 3D reconstruction.}
Classical and learning-based 3D streaming pipelines have progressed from short-horizon depth/pose models
to online recurrent reconstruction. Representative geometric estimators include
RAFT-style correlation, DeepV2D, and DROID-SLAM~\cite{teed2020raft,teed2020deepv2d,teed2021droidslam}.
For online scene understanding, methods such as iMAP, NICE-SLAM, Co-SLAM, Point-SLAM, and MASt3R-SLAM
maintain a persistent latent map updated as frames arrive~\cite{sucar2021imap,zhu2022niceslam,wang2023coslam,sandstrom2023pointslam,murai2025mast3rslam}.
Recent foundation-style 3D models further improve feed-forward reconstruction quality,
from DUSt3R/MASt3R to larger multi-view architectures including
VGGT and its follow-ups~\cite{wang2024dust3r,leroy2024mast3r,wang2025vggt,wang2025pi3,lin2025depthanything3,yang2025fast3r,shen2025fastvggt,deng2025vggtlong}.
Streaming variants such as Spann3R, CUT3R, SLAM3R, Point3R, Stream3R, StreamVGGT, and WinT3R
explicitly target long-horizon online inference under bounded memory~\cite{wang2024spann3r,wang2025cut3r,liu2025slam3r,wu2025point3r,lan2025stream3r,zhuo2025streamvggt,li2025wint3r}.

A central design choice  is the state update rule in recurrent memories.
CUT3R uses uniform overwrite~\cite{wang2025cut3r}, while recent training-free methods
introduce adaptive gating signals from attention or motion cues, including
TTT3R, TTSA3R, and MUT3R~\cite{chen2026ttt3r,zheng2026ttsa3r,shen2025mut3r}.
In parallel, test-time adaptation methods update model parameters online, \eg, test-time
training and entropy minimization~\cite{sun2020ttt,wang2021tent}.

In streaming 3D reconstruction, errors compound with rollout length, causing drift and forgetting
beyond the training horizon. This failure mode is visible in overwrite-based and heuristic-gated
recurrent updates~\cite{wang2025cut3r,chen2026ttt3r,zheng2026ttsa3r}.
Our goal is to address this by explicitly propagating uncertainty so gains decay in stable periods
and rise only when estimated process uncertainty indicates genuine scene change.

\paragraph{Uncertainty estimation and filtering.}
Kalman-style filtering provides a principled framework for sequential estimation under linear-Gaussian assumptions~\cite{kalman1960}.
Adaptive Kalman-style filtering estimates noise covariances online when dynamics are non-stationary~\cite{mehra1970akf}.
Neural and differentiable variants include Deep Kalman Filters and Backprop KF~\cite{krishnan2015deepkf,haarnoja2016backpropkf}.
We adopt this perspective for latent token memories, using adaptive process noise and propagated uncertainty
to derive per-token gains while keeping measurement noise fixed.

\section{FILT3R}

\begin{figure}[t]
    \centering
    \includegraphics[width=0.86\linewidth]{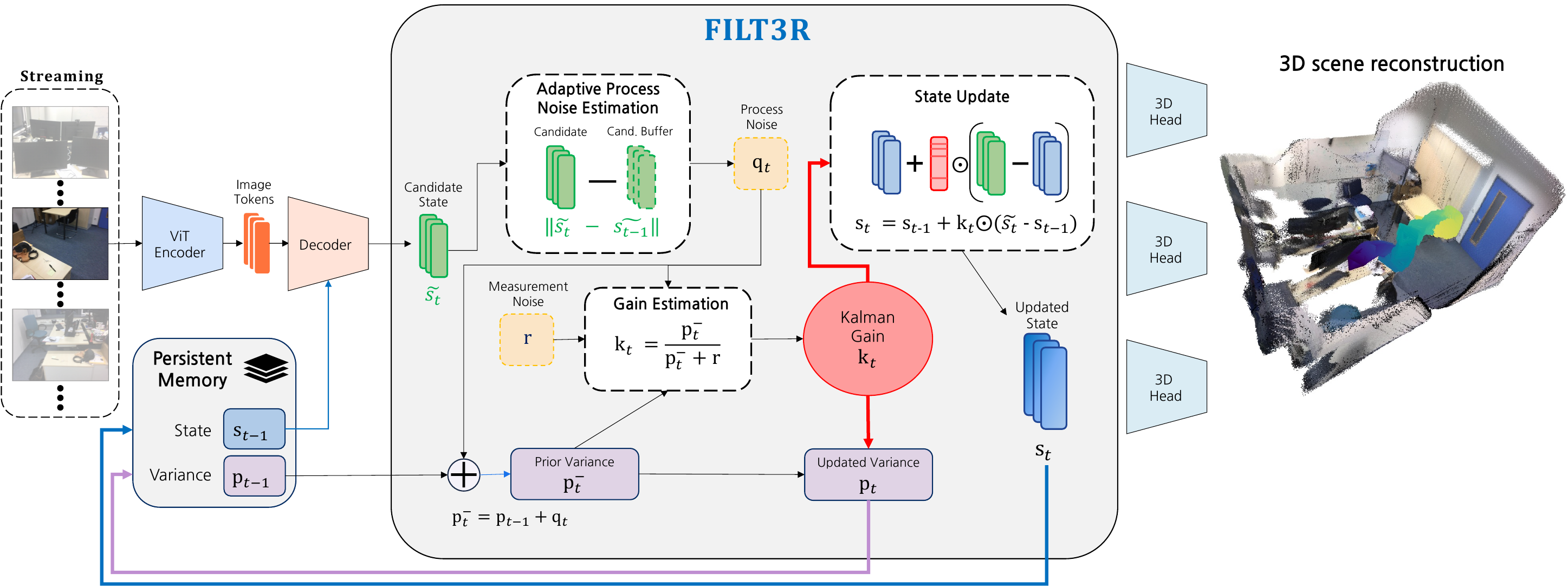}
    \caption{\textbf{FILT3R overview.}
    The persistent memory $\mathbf{s}_{t-1}$ is fused with the decoder's candidate
    $\tilde{\mathbf{s}}_t$ via a token-wise Kalman-style gain $\mathbf{k}_t$.
    Process noise $\mathbf{q}_t$ is estimated adaptively from EMA-normalized temporal drift,
    while measurement noise $r$ is a scalar hyperparameter shared across tokens.
    Variance $\mathbf{p}_t$ is propagated across time steps,
    enabling gains that naturally shrink in stable regimes and increase during scene change.}
    \label{fig:method}
\end{figure}

\subsection{Problem formulation}
\label{sec:problem}

We consider a recurrent streaming reconstruction model that maintains persistent state tokens
$\mathbf{s}_t \in \mathbb{R}^{N \times D}$ and, given an incoming frame at time $t$,
produces a candidate state token $\tilde{\mathbf{s}}_t \in \mathbb{R}^{N \times D}$ using a decoder that attends to image tokens
and the previous state~\cite{wang2025cut3r}. Here, $N$ and $D$ refer to the number of tokens and token dimension, respectively.

Specifically, we interpret $\mathbf{s}_t$ as a \emph{belief} over the scene latent at time $t$ and
$\tilde{\mathbf{s}}_t$ as a noisy \emph{measurement} of that belief produced by the decoder.
This suggests modeling the dynamics with a stochastic state-space model in token space:
\begin{align}
\mathbf{s}_t &= \mathbf{s}_{t-1} + \mathbf{w}_t,
  \qquad \mathbf{w}_t \sim \mathcal{N}(\mathbf{0}, \mathbf{Q}_t), \label{eq:ssm_process}\\
\tilde{\mathbf{s}}_t &= \mathbf{s}_t + \mathbf{v}_t,
  \qquad \mathbf{v}_t \sim \mathcal{N}(\mathbf{0}, \mathbf{R}), \label{eq:ssm_meas}
\end{align}
where $\mathbf{Q}_t$ models uncertainty in latent evolution (process noise) and
$\mathbf{R}$ models uncertainty in the decoder's prediction (measurement noise).

Streaming video is non-stationary: the scene can change rapidly during motion and remain
static during pauses.
We therefore adopt an \emph{adaptive Kalman-style filtering} (AKF)
approach~\cite{mehra1970akf},
estimating process noise $\mathbf{Q}_t$ online from the model's internal representations.

A key design choice in FILT3R is to keep measurement noise \emph{fixed} using
a single scalar shared across tokens, {\em i.e.}
$\mathbf{R} = r\,\mathbf{I}$.
The rationale is that the decoder is a frozen pretrained network: its prediction quality is
determined by its architecture and training, not by the current scene dynamics.
In classical AKF, this corresponds to the standard assumption that the sensor noise is
known~\cite{mehra1970akf}.

In contrast, we found that making both $\mathbf{Q}_t$ and $\mathbf{R}_t$ adaptive introduces the risk of
correlated noise estimates: scene transitions simultaneously increase temporal drift (which
feeds $\mathbf{Q}_t$) and attention novelty (which would feed $\mathbf{R}_t$).
When both spike together, the gain $\mathbf{k}_t$ can overshoot, producing the same instability
that heuristic gates suffer from.
A fixed $r$ in our FILT3R provides a stable anchor, and all adaptivity enters through the process noise.

\subsection{Update rule}
\label{sec:update}

We assume a diagonal per-token variance $\mathbf{p}_t \in \mathbb{R}^{N}$
(one scalar per token), approximating each token's covariance as isotropic.
Equivalently, we use diagonal covariance approximations
$\mathbf{Q}_t=\mathrm{diag}(\mathbf{q}_t)$ and $\mathbf{P}_t=\mathrm{diag}(\mathbf{p}_t)$,
which yields elementwise Kalman updates.
Then, the variance prediction follows the standard Kalman prediction step:
\begin{equation}
\mathbf{p}^{-}_t = \mathbf{p}_{t-1} + \mathbf{q}_t,
\label{eq:p_pred}
\end{equation}
where $\mathbf{q}_t \in \mathbb{R}^N$ is the per-token process noise
estimated online (\cref{sec:adaptive_q}),
and
the Kalman-style gain is given by
\begin{equation}
\mathbf{k}_t = \frac{\mathbf{p}^{-}_t}{\mathbf{p}^{-}_t + r },
\label{eq:kalman_gain}
\end{equation}

Consequently, the state update according to AKF is a gain-weighted interpolation:
\begin{equation}
\mathbf{s}_t = \mathbf{s}_{t-1} + \mathbf{k}_t \odot
  (\tilde{\mathbf{s}}_t - \mathbf{s}_{t-1}).
\label{eq:state_update_main}
\end{equation}
and  the variance is updated using the scalar Joseph form:
\begin{equation}
\mathbf{p}_t = (\mathbf{1}-\mathbf{k}_t)^2 \odot \mathbf{p}^{-}_t
             + r\mathbf{k}_t^2.
\label{eq:p_update}
\end{equation}
where $\odot$ refers to the Hadamard product,
and $\mathbf{k}_t^2= \mathbf{k}_t\odot\mathbf{k}_t$.

\paragraph{Uncertainty modeling of latent evolution.}
\label{sec:adaptive_q}

Given the update rules for $\mathbf{s}_t, \mathbf{p}_t$, AKF formulation
requires the process noise variance that
captures how much the scene latent is expected to change between frames.
In a static scene, $\mathbf{q}_t$ should be small; during rapid motion or scene transitions,
$\mathbf{q}_t$ should be large.
FILT3R estimates this from the temporal drift of consecutive candidate states, using a
per-stream EMA baseline to calibrate the scale.

Specifically,
we model the $i$-th token process noise through a sigmoid gate:
\begin{equation}
q_{t,i} = q_{\min} + (q_{\max} - q_{\min}) \cdot
  \sigma\!\big(\alpha_q(g_{t,i} - \tau_q)\big),
\label{eq:q_gate_main}
\end{equation}
where $q_{\min}$, $q_{\max}$ bound the process noise,
$\alpha_q$ controls the sharpness of the transition,
and $\tau_q$ is the midpoint,
and $g_{t,i}$ refers to the normalized drift score with respect to the running EMA baseline of the stream-level mean drift
$\hat{\Delta}_t$:
\begin{align}
g_{t,i} &= \Delta_{t,i}/\hat{\Delta}_t  \label{eq:g} \\
 \hat{\Delta}_t & = (1-\lambda_\Delta)\,\hat{\Delta}_{t-1}
              + \lambda_\Delta \bar{\Delta}_t  \label{eq:ema_drift},
\end{align}
where  $\lambda_\Delta$ is the EMA rate and
\begin{equation}
\bar{\Delta}_t :=\sum_i \Delta_{t,i}/N,  \quad\mbox{where}\quad
\Delta_{t,i} = \|\tilde{\mathbf{s}}_{t,i} - \tilde{\mathbf{s}}_{t-1,i}\|_2.
\label{eq:drift}
\end{equation}
The normalized drift score  measures whether the current drift is high or low \emph{relative to the stream's
running baseline}.

Intuitively, when the drift is typical (near the EMA baseline), the sigmoid is near its midpoint
and $q_{t,i}$ takes a moderate value;
when the drift is unusually large, $q_{t,i}$ approaches $q_{\max}$, signaling rapid scene change;
when the drift is small, $q_{t,i}$ approaches $q_{\min}$, indicating stability.

\paragraph{Implementation issue.}

For numerical stability, we add small $\epsilon$ to the denominators
of \eqref{eq:kalman_gain} and \eqref{eq:g} to avoid the division by a near zero value.
The gain in \eqref{eq:kalman_gain} is additionally  clamped to  $[k_{\min} , k_{\max} ]$ to prevent complete freezing or full overwrite.
Similarly,
the EMA baseline is clamped to a floor $\Delta_{\mathrm{floor}}$ to avoid division instability.

The overhead of our AKF formulation is minimal: one $N$-dimensional vector of token variances, one per-stream
EMA scalar for normalizing temporal drift, and one $N\times D$ buffer for the previous candidate.

\subsection{Filtering as adaptive token retention}
\label{sec:memory_horizon}

The Kalman-style gain $\mathbf{k}_t$ in \cref{eq:state_update_main} acts as a per-token update
coefficient $\boldsymbol{\beta}_t := \mathbf{k}_t \in [0,1]^N$, yielding an
exponential-smoothing form
$\mathbf{s}_t = (1-\boldsymbol{\beta}_t)\odot\mathbf{s}_{t-1}
+ \boldsymbol{\beta}_t \odot \tilde{\mathbf{s}}_t$.
Unlike overwrite ($\boldsymbol{\beta}_t \!\equiv\! \mathbf{1}$) or fixed/heuristic gates,
FILT3R propagates uncertainty
(\cref{eq:p_pred,eq:p_update}), so gains \emph{naturally shrink} as the state becomes
confident in stable regimes, extending the effective memory horizon,
while still producing gain increases when inferred process noise favors adaptation.

More specifically,
the Kalman-style gain (\cref{eq:kalman_gain}) is entirely driven by the ratio
$\mathbf{p}^{-}_t / (\mathbf{p}^{-}_t + r)$.
In stable regimes, $\mathbf{q}_t$ stays near $q_{\min}$, so
$\mathbf{p}^{-}_t$ gradually decreases as the Joseph update
(\cref{eq:p_update}) shrinks the variance, causing $\mathbf{k}_t$ to decay and
the effective memory horizon to grow.
During scene transitions, $\mathbf{q}_t$ spikes,
$\mathbf{p}^{-}_t$ increases, and $\mathbf{k}_t$ rises, allowing the filter to rapidly
integrate new evidence.
With $\mathbf{q}_t$ driving $\mathbf{p}_t^-$, the gain in \cref{eq:kalman_gain} increases under large drift
and decreases as uncertainty contracts in stable regimes.

This formulation clarifies the relationship to prior update rules:
\begin{itemize}
\item \emph{Uniform overwrite} (CUT3R): $\boldsymbol\beta_t \equiv \mathbf{1}$.
  Equivalent to $r = 0$ (infinite trust in the measurement).
\item \emph{Heuristic gates} (TTT3R): $\boldsymbol\beta_t$ computed from instantaneous
  attention/change statistics without uncertainty propagation.
  Equivalent to resetting variance to a constant each step.
\item \emph{Fixed interpolation}: $\boldsymbol\beta_t \equiv \beta$.
  A single point on the steady-state $q/r$ curve (see Appendix~\ref{app:memory_horizon}).
\end{itemize}

In fact, in FILT3R,  $\boldsymbol\beta_t$ is derived via recursive uncertainty updates
  (\cref{eq:p_pred,eq:kalman_gain,eq:p_update}). This
  yields Kalman-style gains that decay at a rate of  $\mathcal{O}(1/t)$ in static scenes
  (see Appendix~\ref{app:memory_horizon}, Prop.~1) but spike in response to process noise signaling a scene change.
  To support this mechanism, Appendix~\ref{app:memory_horizon} provides a formal unrolling demonstrating that the effective memory horizon expands as evidence accumulates. It further includes a rigorous proof of the $1/t$ gain decay in static environments and a proposition establishing the relationship between the steady-state gain and the noise ratio $q/r$.

\section{Experiments}
\label{sec:experiments}

\begin{figure}[!t]
    \centering
    \includegraphics[width=0.98\linewidth]{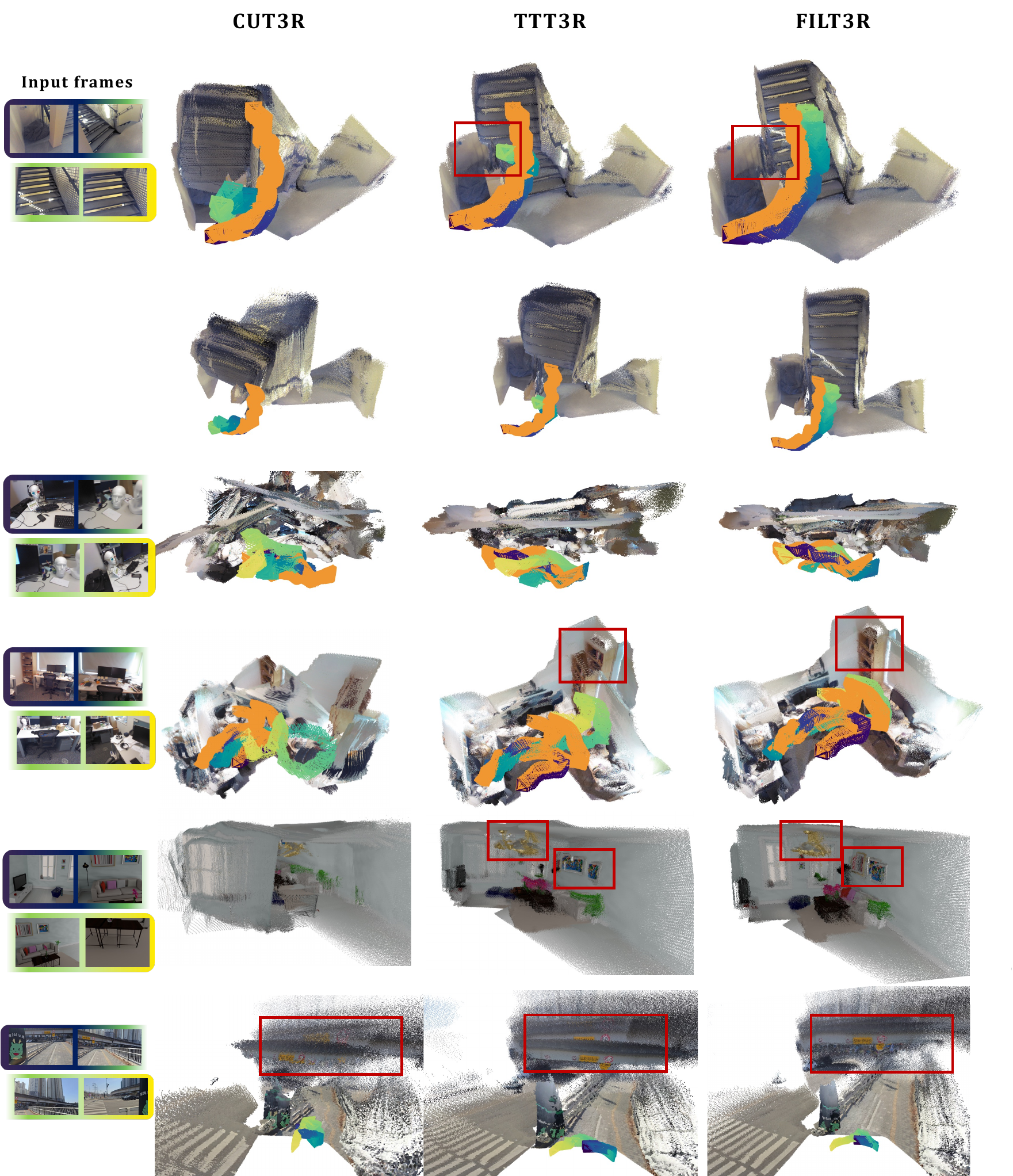}
    \caption{\textbf{Qualitative long-horizon streaming 3D reconstruction.}
CUT3R often suffers from catastrophic forgetting and drift, while TTT3R reduces but does not eliminate long-horizon instability, leading to fragmented surfaces and inconsistent revisited regions (red boxes).
FILT3R produces more coherent geometry over long rollouts.
Ground-truth trajectory is shown in orange; the predicted trajectory is color-coded by time.}
    \label{fig:main_qualitative}
\end{figure}

\subsection{Setup}

\paragraph{Model and update rules.}
We evaluate FILT3R as a drop-in replacement for the state update in a CUT3R-style recurrent architecture~\cite{wang2025cut3r}. All methods share identical backbone weights; only the online update policy differs. FILT3R uses the token-wise filtering update in \cref{eq:state_update_main} with adaptive process noise \cref{eq:q_gate_main} and gain computation in \cref{eq:kalman_gain}.

\paragraph{Tasks and metrics.}
We report (i) video depth estimation with Abs Rel, $\delta{<}1.25$, and log RMSE, (ii) camera pose estimation with ATE, RPE-translation, and RPE-rotation, and (iii) 3D reconstruction with Accuracy, Completeness, and Normal Consistency. For depth, we report both metric-scale and per-sequence scale-aligned metrics following common practice. For long-horizon benchmarks, we additionally visualize error versus sequence length to characterize drift.

\paragraph{Datasets.}
Short-horizon depth is evaluated on Sintel~\cite{butler2012sintel}, Bonn~\cite{palazzolo2019refusion}, and KITTI~\cite{geiger2012kitti} following prior protocols. Long-horizon behavior is evaluated on extended TUM~\cite{sturm2012tum} and Bonn~\cite{palazzolo2019refusion} sequences substantially longer than the training horizon. 3D reconstruction quality is measured on 7-Scenes~\cite{20137scenes} and NRGBD~\cite{nrgbd2022}. Full dataset and preprocessing details are provided in the appendix.

\paragraph{Baselines.}
We compare against CUT3R~\cite{wang2025cut3r} (uniform overwrite) and TTT3R~\cite{chen2026ttt3r} (heuristic confidence scaling). We additionally report streaming reconstruction baseline with an explicit spatial memory architecture (Point3R). Additional comparisons with TTT3R+Reset are provided in the appendix (Tab.~\ref{tab:pose_long_reset} and Tab.~\ref{tab:depth_long_reset}).

\paragraph{Hyperparameter selection protocol.}
FILT3R has a small set of scalar hyperparameters controlling process and measurement noise. To avoid per-benchmark tuning, we tune on a held-out development set (long-horizon TUM-200 for pose, and a short-horizon set for depth) and use the same hyperparameters across all the experiments.The list of hyperparameters is provided in the appendix

\paragraph{Origin-aligned ATE.}
Standard ATE applies a global alignment over the full trajectory, which can partially mask progressive drift.
We therefore also report $\text{ATE}_{\text{orig}}$, aligning only the first frame and evaluating the remaining trajectory without further alignment,
so drift accumulation is directly reflected in the error.

\subsection{Long-horizon 3D reconstruction}
\begin{table*}[t]
\centering
\caption{\textbf{3D reconstruction on 7-Scenes and NRGBD.}
Accuracy (Acc), Completeness (Comp), and Normal Consistency (NC);
Mean and Median (Med.) are reported.
Point3R runs out of memory (OOM) at length 1000.}
\label{tab:recon_main}

{\small (a) \textbf{7-Scenes} (lengths 300 / 500 / 1000)}
\vspace{1pt}

\resizebox{\textwidth}{!}{
\begin{tabular}{l | cccccc | cccccc | cccccc}
\toprule
\multirow{3}{*}{\textbf{Method}} & \multicolumn{6}{c|}{\textbf{Length 300}} & \multicolumn{6}{c|}{\textbf{Length 500}} & \multicolumn{6}{c}{\textbf{Length 1000}} \\
\cmidrule(lr){2-7} \cmidrule(lr){8-13} \cmidrule(lr){14-19}
 & \multicolumn{2}{c}{Acc$\downarrow$} & \multicolumn{2}{c}{Comp$\downarrow$} & \multicolumn{2}{c|}{NC$\uparrow$} & \multicolumn{2}{c}{Acc$\downarrow$} & \multicolumn{2}{c}{Comp$\downarrow$} & \multicolumn{2}{c|}{NC$\uparrow$} & \multicolumn{2}{c}{Acc$\downarrow$} & \multicolumn{2}{c}{Comp$\downarrow$} & \multicolumn{2}{c}{NC$\uparrow$} \\
\cmidrule(lr){2-3} \cmidrule(lr){4-5} \cmidrule(lr){6-7} \cmidrule(lr){8-9} \cmidrule(lr){10-11} \cmidrule(lr){12-13} \cmidrule(lr){14-15} \cmidrule(lr){16-17} \cmidrule(lr){18-19}
 & Mean & Med. & Mean & Med. & Mean & Med. & Mean & Med. & Mean & Med. & Mean & Med. & Mean & Med. & Mean & Med. & Mean & Med. \\
\midrule
CUT3R              & 0.134 & 0.090 & 0.076 & 0.041 & 0.544 & 0.565 & 0.183 & 0.133 & 0.093 & 0.036 & 0.530 & 0.544 & 0.233 & 0.170 & 0.101 & 0.017 & 0.511 & 0.515 \\
Point3R            & 0.045 & 0.025 & 0.032 & 0.013 & 0.564 & 0.598 & 0.057 & 0.026 & 0.025 & 0.008 & 0.556 & 0.584 & \multicolumn{6}{c}{\cellcolor{lightred}{\textbf{OOM}}} \\
TTT3R              & 0.040 & 0.025 & 0.024 & 0.005 & 0.566 & 0.602 & 0.066 & 0.038 & 0.032 & 0.006 & 0.551 & 0.576 & 0.145 & 0.092 & 0.051 & 0.010 & 0.525 & 0.536 \\
\oursrow
\oursname          & \textbf{0.020} & \textbf{0.009} & \textbf{0.022} & \textbf{0.004} & \textbf{0.568} & \textbf{0.603} & \textbf{0.024} & \textbf{0.011} & \textbf{0.023} & \textbf{0.004} & \textbf{0.559} & \textbf{0.589} & \textbf{0.054} & \textbf{0.028} & \textbf{0.028} & \textbf{0.005} & \textbf{0.535} & \textbf{0.550} \\
\bottomrule
\end{tabular}}

\vspace{4pt}
{\small (b) \textbf{NRGBD} (lengths 300 / 400 / 500)}
\vspace{1pt}

\resizebox{\textwidth}{!}{
\begin{tabular}{l | cccccc | cccccc | cccccc}
\toprule
\multirow{3}{*}{\textbf{Method}} & \multicolumn{6}{c|}{\textbf{Length 300}} & \multicolumn{6}{c|}{\textbf{Length 400}} & \multicolumn{6}{c}{\textbf{Length 500}} \\
\cmidrule(lr){2-7} \cmidrule(lr){8-13} \cmidrule(lr){14-19}
 & \multicolumn{2}{c}{Acc$\downarrow$} & \multicolumn{2}{c}{Comp$\downarrow$} & \multicolumn{2}{c|}{NC$\uparrow$} & \multicolumn{2}{c}{Acc$\downarrow$} & \multicolumn{2}{c}{Comp$\downarrow$} & \multicolumn{2}{c|}{NC$\uparrow$} & \multicolumn{2}{c}{Acc$\downarrow$} & \multicolumn{2}{c}{Comp$\downarrow$} & \multicolumn{2}{c}{NC$\uparrow$} \\
\cmidrule(lr){2-3} \cmidrule(lr){4-5} \cmidrule(lr){6-7} \cmidrule(lr){8-9} \cmidrule(lr){10-11} \cmidrule(lr){12-13} \cmidrule(lr){14-15} \cmidrule(lr){16-17} \cmidrule(lr){18-19}
 & Mean & Med. & Mean & Med. & Mean & Med. & Mean & Med. & Mean & Med. & Mean & Med. & Mean & Med. & Mean & Med. & Mean & Med. \\
\midrule
CUT3R              & 0.216 & 0.122 & 0.073 & 0.017 & 0.583 & 0.629 & 0.293 & 0.212 & 0.103 & 0.042 & 0.561 & 0.592 & 0.313 & 0.238 & 0.141 & 0.045 & 0.558 & 0.585 \\
Point3R      & \textbf{0.065} & 0.037 & \textbf{0.012} & \textbf{0.003} & 0.617 & 0.693 & \textbf{0.084} & 0.040 & \textbf{0.020} & \textbf{0.003} & 0.613 & 0.684 & 0.112 & 0.046 & \textbf{0.028} & \textbf{0.003} & 0.614 & 0.687 \\
TTT3R              & 0.102 & 0.044 & 0.024 & 0.004 & 0.613 & 0.684 & 0.144 & 0.064 & 0.072 & 0.011 & 0.598 & 0.656 & 0.170 & 0.095 & 0.090 & 0.019 & 0.590 & 0.642 \\
\oursrow
\oursname      & 0.079 & \textbf{0.027} & \textbf{0.012} & 0.004 & \textbf{0.628} & \textbf{0.710} & 0.086 & \textbf{0.036} & 0.026 & 0.005 & \textbf{0.623} & \textbf{0.701} & \textbf{0.096} & \textbf{0.042} & 0.035 & 0.005 & \textbf{0.621} & \textbf{0.697} \\
\bottomrule
\end{tabular}}
\end{table*}

\Cref{tab:recon_main} reports 3D reconstruction quality on 7-Scenes and NRGBD at long sequence lengths.
On 7-Scenes, FILT3R consistently outperforms CUT3R and TTT3R, with the gap widening as sequences lengthen:
mean accuracy improves from 0.040 to 0.020 at length 300 and from 0.066 to 0.024 at length 500
compared to TTT3R, while maintaining strong completeness and normal consistency.
At length 1000---well beyond the training horizon---Point3R runs out of memory,
while CUT3R and TTT3R degrade substantially; FILT3R continues to produce coherent geometry,
confirming that uncertainty propagation compounds in value over long rollouts.
On NRGBD, FILT3R improves over CUT3R and TTT3R at all lengths and remains competitive with prior streaming baselines at shorter horizons.

\subsection{Long-horizon camera pose estimation}

\begin{table*}[t]
\centering
\caption{\textbf{Long-horizon camera pose on TUM-RGBD.} We report standard ATE (with global Procrustes alignment) and origin-aligned ATE ($\text{ATE}_{\text{orig}}$), which aligns only the first frame and therefore exposes progressive drift. FILT3R provides clear long-horizon stability gains.}
\label{tab:pose_long}
\resizebox{\textwidth}{!}{
\begin{tabular}{l | cccc | cccc | cccc}
\toprule
\multirow{2}{*}{\textbf{Method}} & \multicolumn{4}{c|}{\textbf{TUM-400}} & \multicolumn{4}{c|}{\textbf{TUM-600}} & \multicolumn{4}{c}{\textbf{TUM-800}}\\
\cmidrule(lr){2-5}\cmidrule(lr){6-9}\cmidrule(lr){10-13}
 & ATE$\downarrow$ & $\text{ATE}_{\text{orig}}\!\downarrow$ & RPE-t$\downarrow$ & RPE-r$\downarrow$
 & ATE$\downarrow$ & $\text{ATE}_{\text{orig}}\!\downarrow$ & RPE-t$\downarrow$ & RPE-r$\downarrow$
 & ATE$\downarrow$ & $\text{ATE}_{\text{orig}}\!\downarrow$ & RPE-t$\downarrow$ & RPE-r$\downarrow$ \\
\midrule
CUT3R         & 0.109 & 0.302 & 0.010 & 0.381 & 0.145 & 0.425 & \textbf{0.008} & 0.430 & 0.173 & 0.493 & \textbf{0.008} & 0.486 \\
TTT3R         & 0.055 & 0.128 & 0.010 & 0.328 & 0.084 & 0.176 & 0.009 & 0.365 & 0.097 & 0.214 & 0.009 & 0.387 \\
\oursrow
\oursname  & \textbf{0.033} & \textbf{0.074} & \textbf{0.009} & \textbf{0.305} & \textbf{0.042} & \textbf{0.082} & 0.009 & \textbf{0.332} & \textbf{0.057} & \textbf{0.107} & 0.009 & \textbf{0.362} \\
\bottomrule
\end{tabular}}
\end{table*}

As shown in \Cref{tab:pose_long}, FILT3R substantially reduces long-horizon trajectory error on extended TUM-RGBD sequences.
The improvements are most pronounced under origin-aligned ATE, which directly exposes progressive drift:
at 800 frames, $\text{ATE}_{\text{orig}}$ drops from 0.214 (TTT3R) and 0.493 (CUT3R) to 0.107 with FILT3R.
Standard ATE also improves (0.097 $\rightarrow$ 0.057 at 800 frames).
Local relative errors are similar to TTT3R in translation (RPE-t $\approx$ 0.009), while rotation error decreases
(0.387 $\rightarrow$ 0.362 at 800 frames), suggesting FILT3R primarily mitigates slow drift and catastrophic forgetting rather
than short-term relative motion.

\subsection{Long-horizon video depth estimation}

\begin{table*}[htbp!]
\centering
\caption{\textbf{Long-horizon video depth on Bonn (metric scale).} Sequences are truncated to the indicated length. FILT3R maintains stable metric depth predictions as sequence length grows, whereas competing methods suffer from accumulating drift.}
\label{tab:depth_long}
\resizebox{\textwidth}{!}{
\begin{tabular}{l | ccc | ccc | ccc}
\toprule
\multirow{2}{*}{\textbf{Method}} &
\multicolumn{3}{c|}{\textbf{Bonn-300}} &
\multicolumn{3}{c|}{\textbf{Bonn-400}} &
\multicolumn{3}{c}{\textbf{Bonn-500}}\\
\cmidrule(lr){2-4}\cmidrule(lr){5-7}\cmidrule(lr){8-10}
 & Abs Rel$\downarrow$ & $\delta{<}1.25\uparrow$ & log RMSE$\downarrow$ & Abs Rel$\downarrow$ & $\delta{<}1.25\uparrow$ & log RMSE$\downarrow$ & Abs Rel$\downarrow$ & $\delta{<}1.25\uparrow$ & log RMSE$\downarrow$ \\
\midrule
CUT3R         & 0.107 & 88.7 & 0.162 & 0.107 & 89.6 & 0.161 & 0.101 & 90.6 & 0.156 \\
TTT3R         & 0.108 & 90.1 & 0.163 & 0.104 & 91.2 & 0.158 & 0.100 & 92.1 & 0.154 \\
\oursrow
\oursname & \textbf{0.093} & \textbf{93.8} & \textbf{0.152} & \textbf{0.090} & \textbf{94.1} & \textbf{0.149} & \textbf{0.089} & \textbf{94.4} & \textbf{0.148} \\
\bottomrule
\end{tabular}}
\end{table*}

\Cref{tab:depth_long} evaluates metric-scale depth on extended Bonn sequences (300--500 frames), beyond the training horizon.
FILT3R achieves the best performance at every truncation length.
Compared to TTT3R, FILT3R reduces Abs Rel from 0.108 to 0.093 at 300 frames and from 0.100 to 0.089 at 500 frames,
while improving $\delta{<}1.25$ from 90.1 to 93.8 and from 92.1 to 94.4.
The consistently lower log RMSE indicates improved robustness to rare but severe outlier predictions that can arise from overwriting a stable latent state with a noisy candidate update.
\Cref{fig:long_horizon_metrics} (bottom) further shows FILT3R maintains lower error across sequence prefixes, indicating reduced drift accumulation.

\begin{figure}[t]
\centering
\begin{minipage}[t]{0.46\linewidth}
\centering
\includegraphics[width=\linewidth]{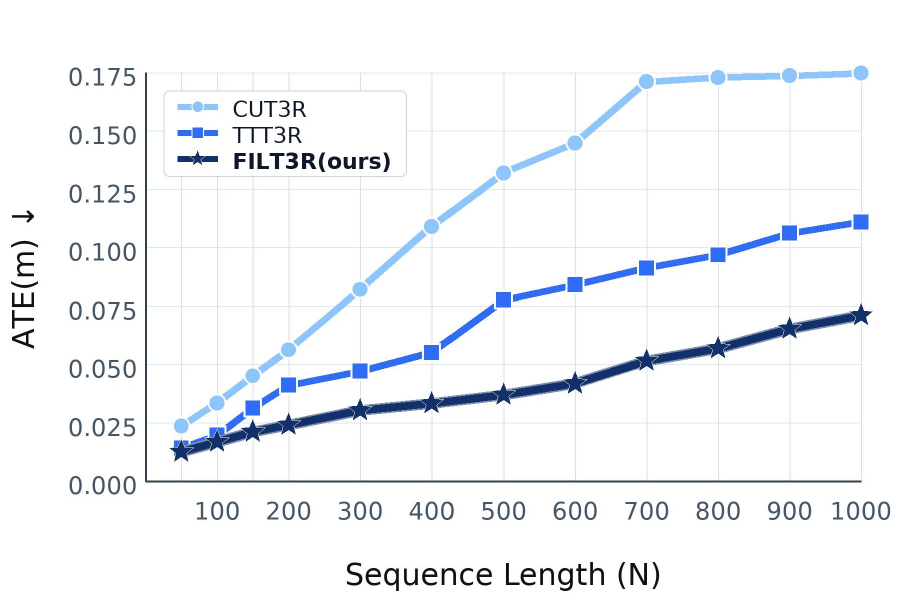}
\caption*{ (a) TUM-RGBD: ATE vs. length }
\end{minipage}
\hfill
\begin{minipage}[t]{0.46\linewidth}
\centering
\includegraphics[width=\linewidth]{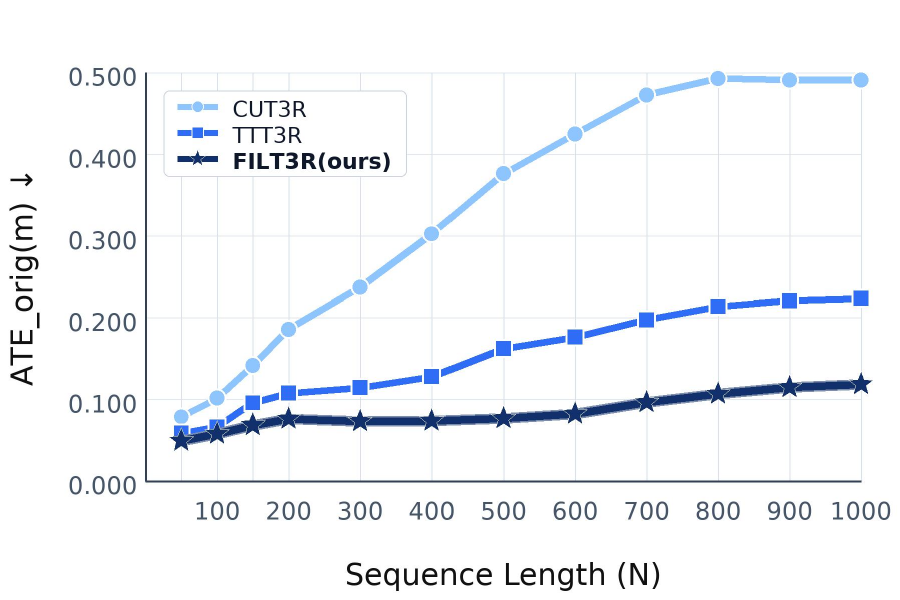}
\caption*{ (b) TUM-RGBD: $\text{ATE}_{\text{orig}}$ vs. length }
\end{minipage}

\begin{minipage}[t]{0.46\linewidth}
\centering
\includegraphics[width=\linewidth]{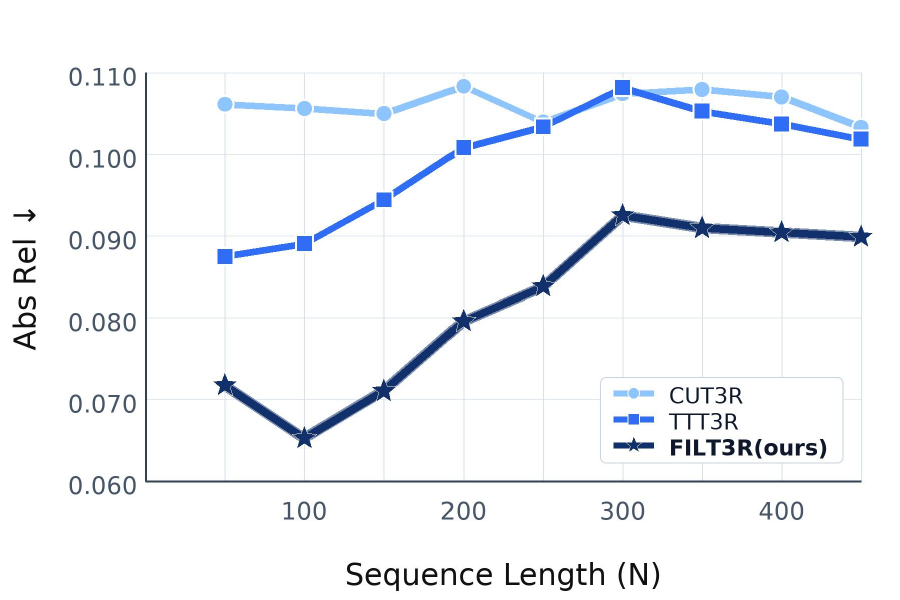}
\caption*{ (c) Bonn: Abs Rel vs. length }
\end{minipage}
\hfill
\begin{minipage}[t]{0.46\linewidth}
\centering
\includegraphics[width=\linewidth]{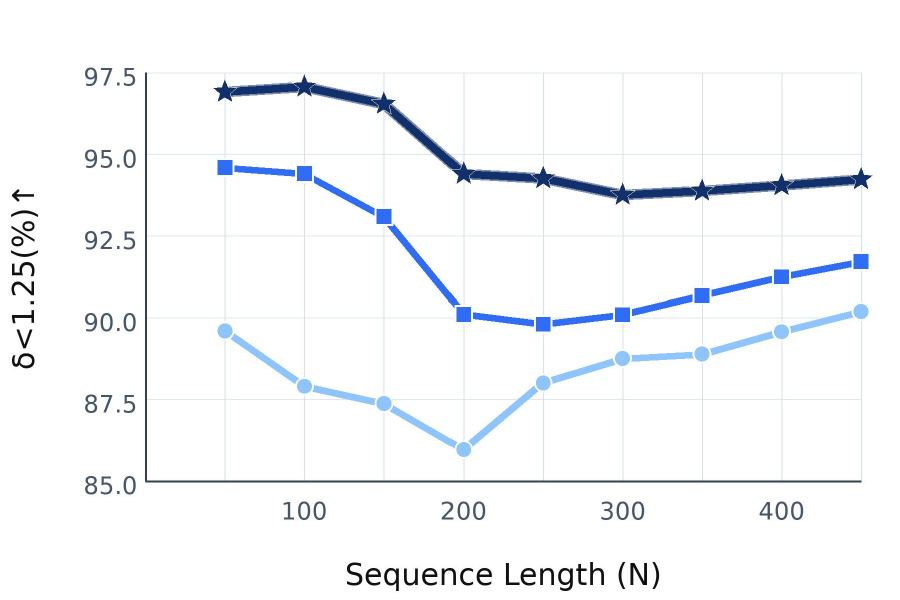}
\caption*{ (d) Bonn: $\delta{<}1.25$ vs. length }
\end{minipage}

\caption{\textbf{Long-horizon metrics vs. sequence length.} \textbf{Top (a, b):} Camera pose on TUM-RGBD (50--1000 frames). FILT3R's drift grows more slowly. \textbf{Bottom (c, d):} Metric scale video depth on Bonn (50--500 frames). FILT3R consistently mitigates drift and improves over the baselines as the sequence progresses.}
\label{fig:long_horizon_metrics}
\end{figure}

\subsection{Generality across memory architectures}
\label{sec:generality}
FILT3R is presented in a CUT3R-style setup, but the underlying idea---treating the recurrent state as a belief and the decoder candidate as a noisy measurement---is not specific to CUT3R's compact latent. We therefore apply it to two streaming models with \emph{different} memory designs, changing only their state-update rule and retraining nothing.
Applied to the fast-weight memory of tttLRM~\cite{wang2026tttlrm} (inserted into its LaCT/Muon transition, denoted FILT3R-FW), FILT3R improves PSNR over the unmodified autoregressive update by up to $+1.08$\,dB at length 128 and, when finer update steps induce catastrophic forgetting, rescues up to $+1.6$\,dB on all 140 DL3DV-140 scenes (\cref{fig:tttlrm_qual}). It also extends to Point3R's~\cite{wu2025point3r} explicit pointer memory and compares favorably to the attention/KV-cache memory InfiniteVGGT~\cite{infinitevggt}, so FILT3R generalizes across \emph{fast-weight}, \emph{explicit-pointer}, and \emph{compact-latent} memories. Full tables and analysis are in Appendix~\ref{app:tttlrm} (fast weights) and Appendix~\ref{app:memory_designs} (pointer memory).

\begin{figure}[t]
\centering
\includegraphics[width=0.85\linewidth]{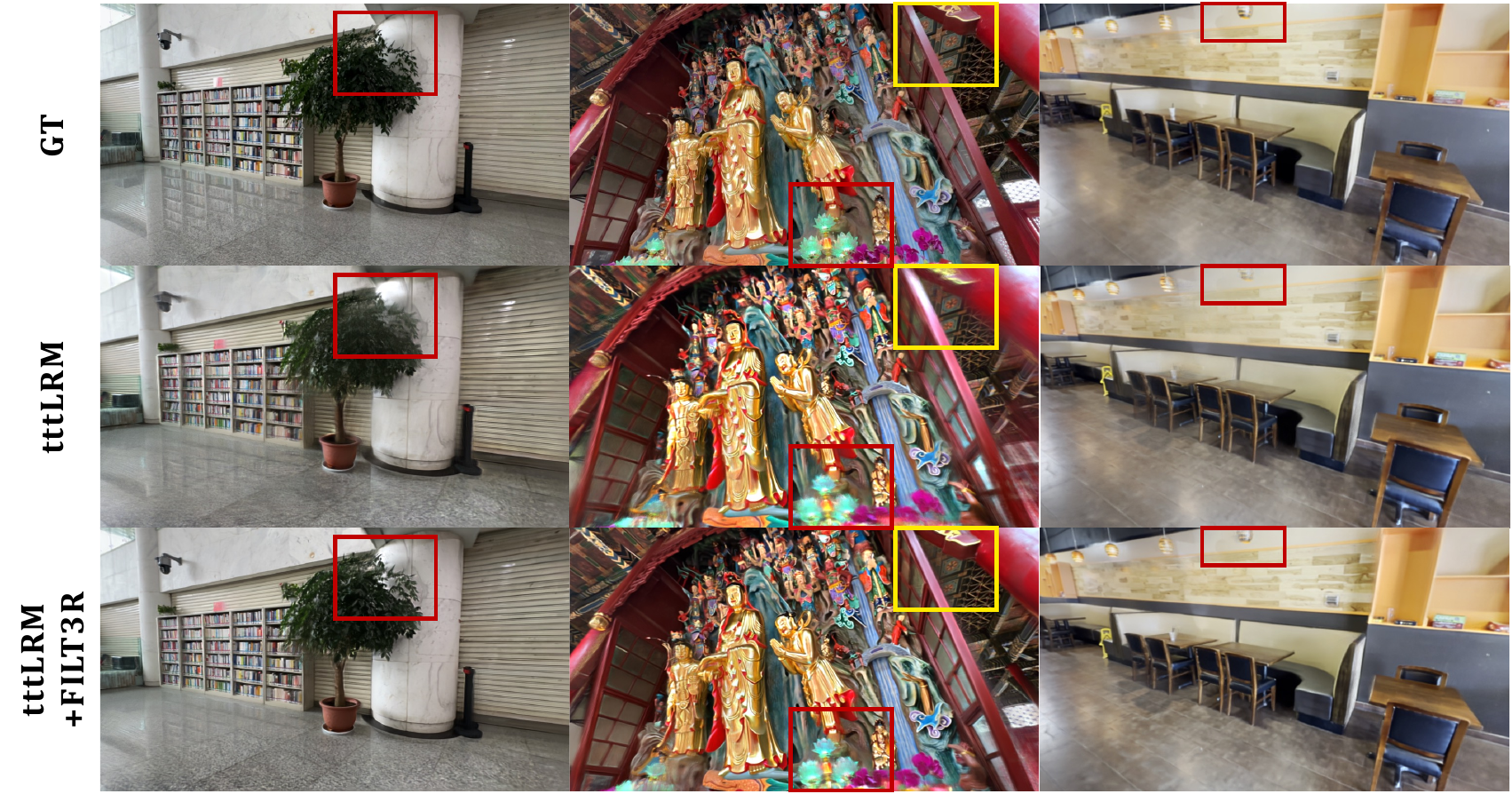}
\caption{\textbf{FILT3R on tttLRM fast-weight memory (DL3DV-140).} Novel-view renders from GT, tttLRM, and tttLRM+FILT3R (top to bottom). Inserting FILT3R into the fast-weight update sharpens fine structure and suppresses the blur and ghosting that the unfiltered model accumulates over long view sequences; highlighted regions mark representative cases.}
\label{fig:tttlrm_qual}
\end{figure}

\subsection{Short-horizon regression check}

Finally, we verify that FILT3R does not sacrifice short-horizon accuracy for long-horizon stability.
Comparing the three update rules under identical backbone weights, FILT3R matches or improves short-horizon accuracy while providing clear long-horizon stability gains.
Under metric-scale evaluation the gains are larger: Abs Rel drops from 1.029 (CUT3R) and 0.977 (TTT3R) to 0.772 on Sintel, and from 0.103 / 0.090 to 0.067 on Bonn.
Full comparisons including optimization-based and full-attention baselines are in Appendix~\ref{app:depth_short_full}.

\subsection{Efficiency}
\label{sec:efficiency}

\begin{figure}[t]
\centering
\begin{minipage}[t]{0.49\linewidth}\vspace{0pt}
    \captionof{table}{\textbf{Runtime and GPU memory (500 frames).}
    Benchmark on NRGBD (512$\times$384, warmup=2, 10 runs).
    \textbf{FILT3R (ours)} avoids attention-map caching, matching CUT3R's footprint.}
    \label{tab:efficiency}
    \begin{center}
    \resizebox{\linewidth}{!}{
    \small
    \begin{tabular}{l c c c}
    \toprule
    \textbf{Method} & \textbf{Attn} & \textbf{FPS}$\uparrow$ & \textbf{Mem}(MB)$\downarrow$ \\
    \midrule
    CUT3R                   & --         & 25.27{\scriptsize$\pm$.05} & 3503 \\
    TTT3R                   & \checkmark & 24.69{\scriptsize$\pm$.05} & 6294 \\
    FILT3R (adpt.\ $r_t$)  & \checkmark & 24.83{\scriptsize$\pm$.04} & 6294 \\
    \oursrow
    \oursname               & --         & 25.14{\scriptsize$\pm$.04} & 3503 \\
    \bottomrule
    \end{tabular}
    }
    \end{center}
\end{minipage}\hfill
\begin{minipage}[t]{0.49\linewidth}\vspace{0pt}
    \centering
    \includegraphics[width=0.85\linewidth]{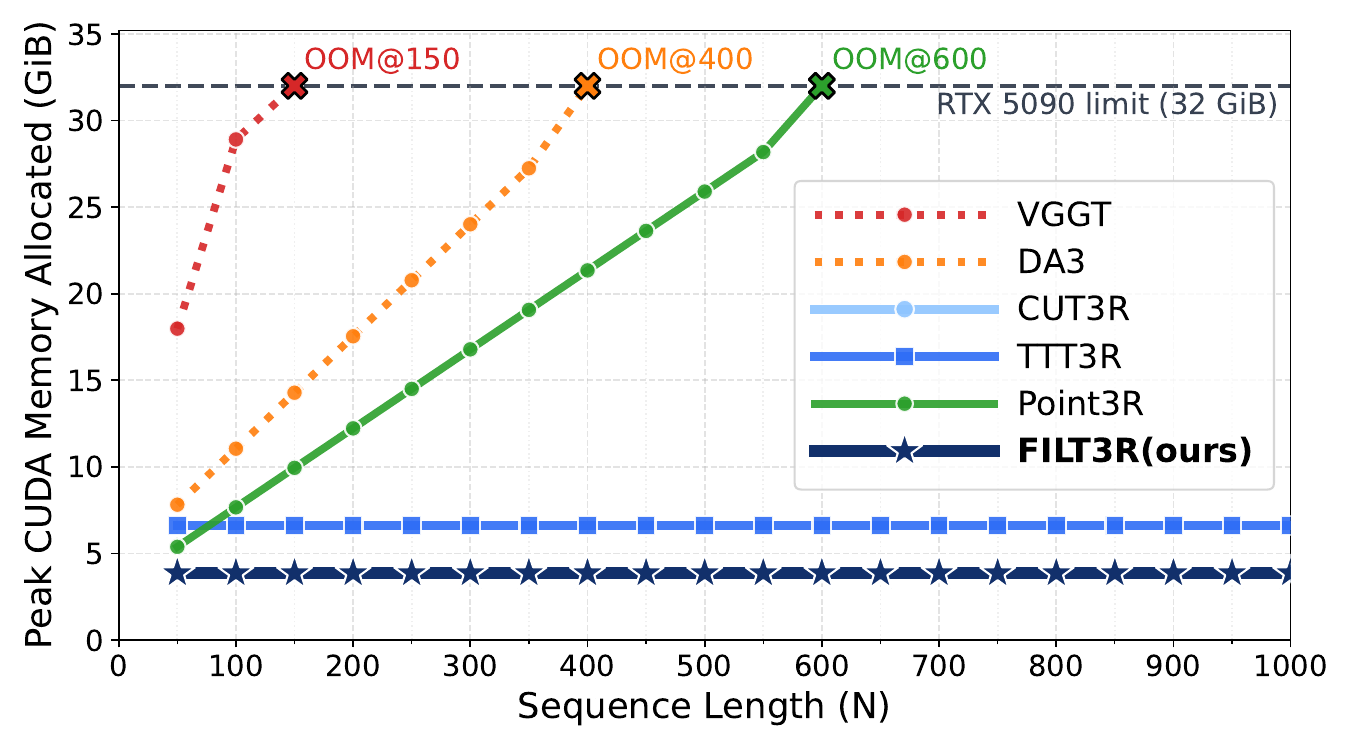}
    \caption{\textbf{Peak GPU mem. vs.\ frames.}
    Full-attention methods (VGGT, DA3) and Point3R grow with length;
    CUT3R/FILT3R stay constant. FILT3R overlaps CUT3R, confirming negligible overhead.}
    \label{fig:memory_vs_frames}
\end{minipage}
\end{figure}

\Cref{tab:efficiency} and \Cref{fig:memory_vs_frames} compare runtime and memory.
Full-attention architectures (VGGT, DA3) and Point3R scale linearly with sequence length, with Point3R running out of memory beyond ${\sim}$600 frames on a 32\,GB GPU.
FILT3R maintains effectively the same constant-memory footprint as CUT3R while TTT3R nearly doubles memory due to cached attention maps for attention-map-based gating.

\subsection{Analysis}
\label{sec:analysis}
To understand why FILT3R improves long-horizon stability, we analyze its induced update dynamics on long TUM-RGBD sequences. FILT3R behaves as an uncertainty-calibrated conservative updater: as confidence accumulates, propagated uncertainty contracts and the effective gain settles into a low-update regime, which reduces the applied update magnitude. This matches the error profile in \Cref{tab:pose_long}: FILT3R substantially lowers $\text{ATE}_{\text{orig}}$ while leaving short-horizon RPE nearly unchanged, indicating that its main benefit is not more aggressive correction at transition frames, but less persistent overwrite of the latent state and therefore slower cumulative drift. The adaptive process-noise term still matters because it temporarily relaxes this conservative regime on difficult frames, allowing the model to remain responsive without reverting to continual high-gain updates; consistent with this interpretation, removing uncertainty propagation in \Cref{tab:ablation} substantially degrades long-horizon performance. Detailed transition-level statistics are provided in the appendix (Fig.~\ref{fig:transition_timeline}).

\subsection{Ablations}
\label{sec:ablations}

\begin{table}[t]
\centering
\caption{\textbf{Ablation study of FILT3R components.}
All variants use identical frozen backbone weights and identical inference code;
only the online update policy is changed. We report long-horizon pose on TUM-800
and long-horizon metric depth on Bonn-500. Lower is better ($\downarrow$) except
$\delta{<}1.25$ ($\uparrow$).}
\label{tab:ablation}
\small
\setlength{\tabcolsep}{4pt}
\resizebox{\textwidth}{!}{
\begin{tabular}{l|cccc|ccc}
\toprule
\multirow{2}{*}{\textbf{Variant}} &
\multicolumn{4}{c|}{\textbf{TUM-800 (pose)}} &
\multicolumn{3}{c}{\textbf{Bonn-500 (metric depth)}}\\
\cmidrule(lr){2-5}\cmidrule(lr){6-8}
& ATE$\downarrow$ & $\text{ATE}_{\text{orig}}\downarrow$ & RPE-t$\downarrow$ & RPE-r$\downarrow$
& AbsRel$\downarrow$ & $\delta{<}1.25\uparrow$ & logRMSE$\downarrow$ \\
\midrule
FILT3R (full)                            & 0.057 & 0.107 & 0.009 & 0.362 & 0.089 & 94.4 & 0.148 \\
Fixed-$\mathbf{Q}$ (fixed $r$)           & 0.100 & 0.229 & 0.009 & 0.393 & 0.099 & 91.9 & 0.155 \\
No variance propagation (reset-$\mathbf{P}$) & 0.149 & 0.489 & 0.009 & 0.446 & 0.099 & 91.1 & 0.154 \\
Fixed-$\beta$ EMA ($\beta{=}0.05$, tuned) & 0.049 & 0.078 & 0.010 & 0.658 & 0.093 & 93.5 & 0.152 \\
No EMA normalization for drift           & 0.058 & 0.109 & 0.010 & 0.367 & 0.091 & 94.1 & 0.150 \\
Adaptive measurement noise ($r_t$)       & 0.074 & 0.148 & 0.009 & 0.372 & 0.093 & 93.9 & 0.151 \\
\bottomrule
\end{tabular}
}
\end{table}
All variants in Table~\ref{tab:ablation} share identical frozen backbone weights and identical inference code;
we only change the \emph{online state update}.
FILT3R’s update is fully determined by the gain $\boldsymbol{\beta}_t=\mathbf{k}_t$ in Eq.~\eqref{eq:state_update_main},
which in turn depends on (i) the process noise $\mathbf{q}_t$ (Eq.~\eqref{eq:q_gate_main}),
(ii) the measurement noise $r$ (Eq.~\eqref{eq:kalman_gain}), and
(iii) whether uncertainty is \emph{propagated} through the variance recursion (Eqs.~\eqref{eq:p_pred}--\eqref{eq:p_update}).
The ablations isolate each of these design choices.

\paragraph{(i) Adaptive process noise is necessary for targeted adaptation.}
\textbf{Fixed-$\mathbf{Q}$} replaces the drift-driven $\mathbf{q}_t$ with a constant.
This removes the ability to temporarily increase the update strength during genuine motion bursts or scene change.
As a result, long-horizon drift grows substantially (\eg, $\text{ATE}_{\text{orig}}$ increases from 0.107 to 0.229),
and depth accuracy also degrades (AbsRel 0.089 $\rightarrow$ 0.099).
This indicates that adaptivity should enter through the \emph{process model} rather than a fixed gate.

\paragraph{(ii) Variance propagation enables confidence accumulation.}
\textbf{No variance propagation (reset-$\mathbf{P}$)} keeps the same gain formula but discards the recursive uncertainty update.
Without accumulated confidence, the gain cannot systematically shrink in stable intervals,
so the model keeps injecting per-frame measurement noise into the state, leading to severe long-horizon drift
($\text{ATE}_{\text{orig}}$ 0.107 $\rightarrow$ 0.489) and worse depth (AbsRel 0.089 $\rightarrow$ 0.099).
This confirms that the Kalman-style \emph{propagation} (not just the functional form of the gain) is important.

\paragraph{(iii) A tuned fixed-gain EMA can reduce ATE, but it harms local motion and depth.}
\textbf{Fixed-$\beta$ EMA} uses a constant update
$\mathbf{s}_t=(1-\beta)\mathbf{s}_{t-1}+\beta\,\tilde{\mathbf{s}}_t$ with $\beta=0.05$ tuned on a development set.
Interestingly, this can yield \emph{lower} trajectory ATE on TUM-800 (0.049 vs.\ 0.057) and also improve
$\text{ATE}_{\text{orig}}$ (0.078 vs.\ 0.107), suggesting that a conservative fixed update can suppress some global drift.
However, it does so by uniformly damping updates, which noticeably degrades \emph{local} relative motion:
rotation error increases sharply (RPE-r 0.362 $\rightarrow$ 0.658).
At the same time, depth quality is still worse than FILT3R.
Overall, fixed-$\beta$ behaves like a \emph{global low-pass filter}: it can reduce drift under some metrics,
but it cannot simultaneously preserve accurate short-term motion and maintain depth quality across long rollouts.
This highlights the need for \emph{uncertainty-aware} and \emph{transition-aware} gains rather than a single fixed forgetting rate.

\paragraph{(iv) EMA normalization is a stabilizer but not the dominant factor here.}
\textbf{No EMA normalization} removes the running baseline $\hat{\Delta}_t$ in Eq.~\eqref{eq:ema_drift}.
On these benchmarks, drift magnitudes are already relatively well-scaled, so the effect is modest
(\eg, $\text{ATE}_{\text{orig}}$ 0.107 $\rightarrow$ 0.109; AbsRel 0.089 $\rightarrow$ 0.091).
We keep EMA normalization as a robust default because it mitigates early/late scale shifts and reduces gain jitter
when drift statistics vary across scenes and sequences.

\paragraph{(v) Adaptive measurement noise hurts long-horizon stability and adds overhead.}
\textbf{Adaptive $r_t$} replaces the fixed scalar $r$ with a per-token measurement noise estimated from attention statistics,
following the intuition that tokens attending strongly to image evidence should be trusted more (lower $r_{t,i}$).
However, this variant degrades long-horizon performance compared to the full model:
$\text{ATE}_{\text{orig}}$ increases from 0.107 to 0.148 and ATE from 0.057 to 0.074,
while depth also worsens (AbsRel 0.089 $\rightarrow$ 0.093, $\delta{<}1.25$ 94.4 $\rightarrow$ 93.9).
We believe the reason is that attention-aligned novelty is correlated with temporal drift:
during scene transitions, both $\mathbf{q}_t$ increases (from large drift) and $r_{t,i}$ decreases (from high attention to novel content),
causing the gain to spike more aggressively than either signal alone would warrant.
This co-excitation amplifies updates precisely when stability matters most, undermining the filter's conservative accumulation of confidence.
Since adaptive $r_t$ additionally requires caching attention maps---nearly doubling GPU memory to match attention-based gating baselines (\cref{tab:efficiency})---we
favor the simpler and more stable design: a fixed scalar $r$ as a measurement-noise anchor,
with all adaptivity concentrated in the process noise $\mathbf{q}_t$.

\section{Limitations and discussion}
\label{sec:limitations}
FILT3R fixes $r$ to a scalar and concentrates adaptivity in $\mathbf{q}_t$, which suits a frozen decoder but ignores genuine measurement-quality variation (blur, occlusion, motion into unseen regions); our ablation shows that estimating $r_t$ from the same novelty signal that drives $\mathbf{q}_t$ \emph{hurts} stability (\cref{sec:ablations}), so reliable $r_t$ adaptation likely needs an \emph{external} cue, and because large drift only \emph{raises} the gain FILT3R can over-trust the candidate when the backbone is least reliable.
Our linear-Gaussian view is an inductive bias rather than an exact model of the nonlinear decoder---hence Kalman-\emph{style}---and it assumes a streaming model exposing a candidate and a previous state; extending it to attention/KV-cache memories without an explicit recurrent state remains open.

\section{Conclusion}
We introduced FILT3R, a principled, training-free solution to the long-horizon drift and catastrophic forgetting that typically degrade streaming 3D reconstruction. By casting recurrent token updates as an adaptive Kalman-style filtering problem, FILT3R explicitly propagates per-token uncertainty to dynamically balance historical memory retention with fast adaptation. Crucially, this uncertainty-aware state fusion acts as a plug-and-play stabilizer. Relying solely on internal temporal drift to estimate process noise—anchored by a single, fixed measurement noise scalar—it unifies and outperforms prior heuristic update rules without requiring any retraining. Our extensive evaluations across video depth, camera pose, and 3D reconstruction demonstrate that FILT3R significantly extends the effective memory horizon, paving the way for highly stable, continuous, and long-form online scene understanding.

\section*{Acknowledgements}
\begingroup\small
This work was supported by the National Research Foundation of Korea under Grant RS-2024-00336454, and by the Institute for Information \& Communications Technology Planning \& Evaluation (IITP) grant funded by the Korea government (MSIT) (RS-2019-II190075, Artificial Intelligence Graduate School Program (KAIST)).
This research was also supported by the Advanced GPU Utilization Support Program funded by the Government of the Republic of Korea (Ministry of Science and ICT) (02-26-01-0404), and by the AI Computing Infrastructure Enhancement (GPU Rental Support) User Support Program funded by the Ministry of Science and ICT (MSIT), Republic of Korea (RQT-25-120217).

\par\endgroup

\clearpage
\bibliographystyle{splncs04}
\bibliography{main}

\clearpage
\appendix
\renewcommand\thesection{\Alph{section}}
\renewcommand\theHsection{appendix.\Alph{section}}

\section*{Appendix Index}
This appendix is organized as follows.

\appendixindexentry{sec:theory}{A. Method, theoretical analysis, and comparison with fixed-EMA smoothing}
\appendixsubindexentry{app:memory_horizon}{A.1 Growth of the effective memory horizon}
\appendixsubindexentry{sec:gain_floor}{A.2 Implemented gain floor}
\appendixsubindexentry{sec:not_ema}{A.3 Distinguishing FILT3R from fixed-EMA smoothing}
\appendixindexentry{app:reset_analysis}{B. Why periodic hard resets are not sufficient}
\appendixindexentry{sec:appendix_protocol}{C. Evaluation protocol and reproducibility}
\appendixindexentry{app:kitti_long}{D. Additional long-horizon depth results on KITTI}
\appendixindexentry{app:depth_short_full}{E. Complete short-horizon tables}
\appendixindexentry{app:hyperparams}{F. Hyperparameter details}
\appendixindexentry{app:qualitative_long}{G. Additional qualitative evidence of long-horizon stability}
\appendixindexentry{app:tttlrm}{H. FILT3R on fast-weight memory (tttLRM): integration and detailed results}
\appendixsubindexentry{app:tttlrm_integration}{H.1 Integration into the LaCT/Muon fast-weight update}
\appendixsubindexentry{tab:tttlrm_len}{H.2 Length scaling at the standard update step}
\appendixsubindexentry{tab:tttlrm_chunks}{H.3 Chunk-axis stress test: catastrophic forgetting and rescue}
\appendixsubindexentry{tab:tttlrm_gain}{H.4 Induced gain dynamics}
\appendixsubindexentry{tab:tttlrm_hparams}{H.5 Hyperparameters}
\appendixindexentry{app:memory_designs}{I. Generality to pointer memory (Point3R)}

\clearpage

\section{Method, theoretical analysis, and comparison with fixed-EMA smoothing}
\label{sec:theory}

\paragraph{Update rule recap.}
FILT3R modifies only the recurrent state-update rule, while backbone weights, decoder weights, and all other aspects of recurrent inference remain unchanged.
Specifically,
let $\tilde{s}_{t,i}$ denote the decoder's candidate token for the $i$-th token at frame $t$, let $p_{t,i}$ denote the propagated per-token variance, and let $r$ be the fixed scalar measurement noise shared across all tokens and experiments.
FILT3R computes a normalized drift score from consecutive candidate states and maps it to per-token process noise via a sigmoid gate:
\begin{equation}
\Delta_{t,i} = \lVert \tilde{s}_{t,i} - \tilde{s}_{t-1,i} \rVert_2,
\qquad
\hat{\Delta}_t = (1-\lambda_\Delta)\hat{\Delta}_{t-1} + \lambda_\Delta \bar{\Delta}_t,
\qquad
\bar{\Delta}_t = \frac{1}{N}\sum_i \Delta_{t,i},
\end{equation}
\begin{equation}
g_{t,i} = \Delta_{t,i}/\hat{\Delta}_t,
\qquad
q_{t,i} = q_{\min} + (q_{\max} - q_{\min}) \sigma\!\big(\alpha_q (g_{t,i} - \tau_q)\big).
\label{eq:q_gate_main_appendix}
\end{equation}
where $\sigma(\cdot)$ denotes the sigmoid.
The resulting process noise $q_{t,i}$ enters the scalar Kalman-style update:
\begin{equation}
p^-_{t,i} = p_{t-1,i} + q_{t,i},
\qquad
k_{t,i} = \frac{p^-_{t,i}}{p^-_{t,i} + r},
\end{equation}
\begin{equation}
s_{t,i} = (1-k_{t,i}) s_{t-1,i} + k_{t,i} \tilde{s}_{t,i},
\qquad
p_{t,i} = (1-k_{t,i})^2 p^-_{t,i} + k_{t,i}^2 r .
\label{eq:joseph}
\end{equation}
The complete per-frame procedure is summarized in Algorithm~\ref{alg:filt3r_update}.

The chain of effects is as follows:
large drift raises $q_{t,i}$, larger $q_{t,i}$ raises the predicted variance $p^-_{t,i}$, larger $p^-_{t,i}$ raises the gain $k_{t,i}$, and a larger gain makes the state move more aggressively toward the new candidate.
Conversely, if the stream remains stable for many frames, then $q_{t,i}$ stays small and the Joseph update in Eq.~\eqref{eq:joseph} contracts the posterior variance $p_{t,i}$, which reduces future gains and makes the latent state increasingly conservative. This is beneficial because the decoder's candidate $\tilde{s}_t$ is a noisy single-frame estimate: in stable regions, the accumulated state is a more reliable scene representation, and conservative updates preserve this advantage.

This is the key distinction from TTT3R.
TTT3R is also adaptive, but its gate is computed directly from current-frame cues (e.g., attention/change statistics) and does not maintain a propagated covariance state.
FILT3R separates these two roles:
$q_{t,i}$ provides frame-dependent adaptivity, while $p_{t,i}$ carries forward accumulated confidence.
As a result, FILT3R's gain depends not only on what is happening in the current frame, but also on how much consistent evidence has already been accumulated across the preceding history.

For numerical stability,
a small $\varepsilon$ is added to the denominators of the gain and normalized drift computations, the gain is clamped to $[k_{\min},k_{\max}]$, and the EMA drift baseline is clamped to a floor $\Delta_{\mathrm{floor}}$.
The additional state beyond the baseline consists of one $N$-dimensional variance vector, one scalar EMA baseline per stream, and one buffer storing the previous candidate state.

\begin{figure}[!ht]
\centering
\begin{minipage}[t]{0.48\linewidth}
\centering
\includegraphics[width=\linewidth]{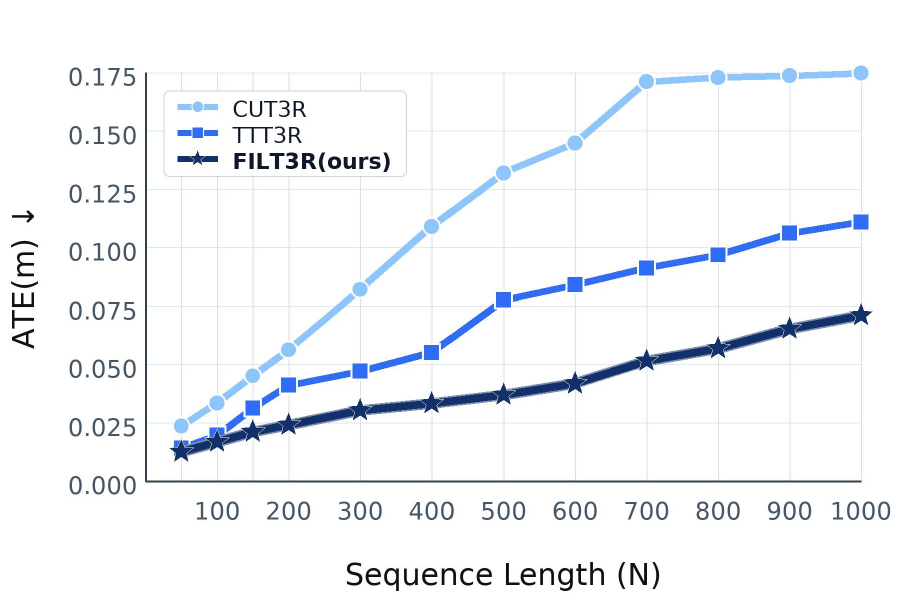}
\caption*{(a) TUM-RGBD: ATE vs. length}
\end{minipage}\hfill
\begin{minipage}[t]{0.48\linewidth}
\centering
\includegraphics[width=\linewidth]{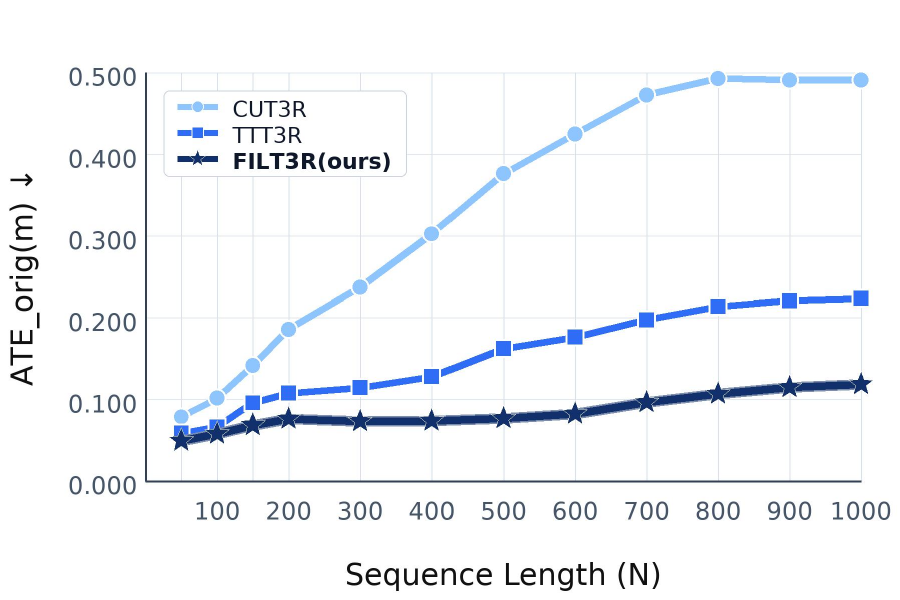}
\caption*{(b) TUM-RGBD: $\text{ATE}_{\text{orig}}$ vs. length}
\end{minipage}

\begin{minipage}[t]{0.48\linewidth}
\centering
\includegraphics[width=\linewidth]{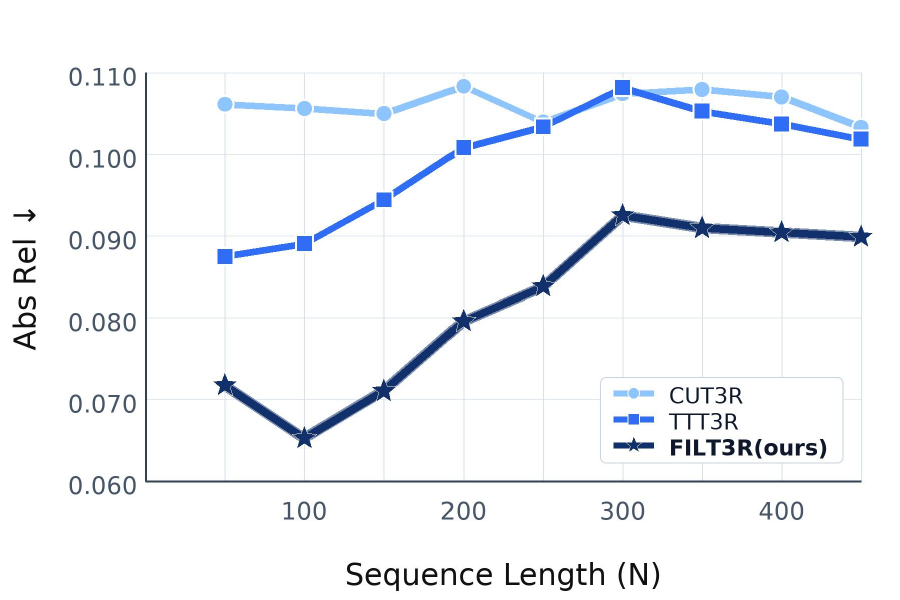}
\caption*{(c) Bonn: Abs Rel vs. length}
\end{minipage}\hfill
\begin{minipage}[t]{0.48\linewidth}
\centering
\includegraphics[width=\linewidth]{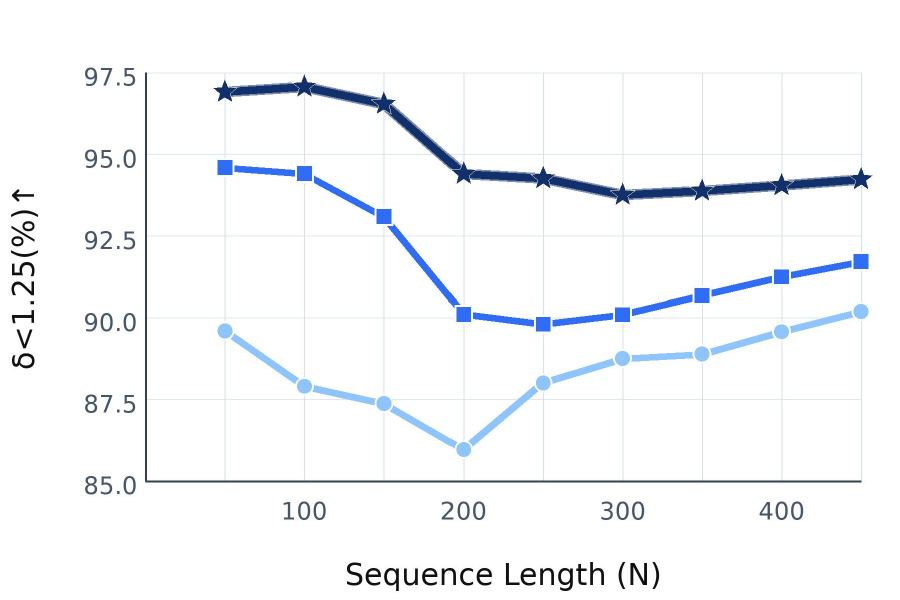}
\caption*{(d) Bonn: $\delta{<}1.25$ vs. length}
\end{minipage}
\caption{\textbf{FILT3R slows long-horizon error growth.}
Across both camera pose and metric-scale depth, FILT3R's errors grow more slowly than those of overwrite (CUT3R) and heuristic gating (TTT3R) as the evaluated prefix length increases.
This is the central empirical pattern that the remainder of the appendix explains.}
\label{fig:long_horizon_metrics_appendix}
\end{figure}

\paragraph{Notation and first-frame handling.}
Notation follows the main paper.
In Sec.~\ref{app:memory_horizon}, we write $\beta_{t,i}:=k_{t,i}$ only to make the smoothing interpretation explicit; all other symbols are unchanged.
On the first valid frame, the latent state is initialized by direct overwrite with the current candidate, the covariance is set to $p_0$ for every token, and the previous-candidate buffer is set to the current candidate state.
The drift EMA is not preseeded separately: when the first temporal difference becomes available, its running value is initialized from that mean drift and then clamped to $\Delta_{\mathrm{floor}}$.

\subsection{Growth of the effective memory horizon}
\label{app:memory_horizon}

The Kalman-style gain $k_{t,i}$ acts as a per-token update coefficient $\beta_{t,i} := k_{t,i} \in [0,1]$, yielding a time-varying exponential smoothing:
\begin{equation}
s_{t,i} = (1-\beta_{t,i}) s_{t-1,i} + \beta_{t,i} \tilde{s}_{t,i}.
\label{eq:ema_form}
\end{equation}
Unrolling this recursion gives a weighted history of past candidates:
\begin{equation}
s_{t,i}
=
\Big(\prod_{u=1}^{t}(1-\beta_{u,i})\Big)s_{0,i}
+
\sum_{\tau=1}^{t}
\Big(
\beta_{\tau,i}\prod_{u=\tau+1}^{t}(1-\beta_{u,i})
\Big)\tilde{s}_{\tau,i}.
\label{eq:unroll_weights}
\end{equation}
A constant gain $\beta_{t,i}\equiv\beta$ gives a fixed exponential memory horizon. If $\beta_{t,i}$ decreases over time, older candidates decay more slowly and the effective memory horizon expands.

\begin{algorithm}[!tb]
\caption{\textbf{FILT3R replaces only the latent-state update.}
The recurrent architecture, decoder, and backbone checkpoint are otherwise identical to the baseline.}
\label{alg:filt3r_update}
\begin{algorithmic}[1]
    \Require Previous latent state $\mathbf{s}_{t-1}$, covariance $\mathbf{p}_{t-1}$, candidate state $\tilde{\mathbf{s}}_{t-1}$, and measurement noise $r$
    \State Decode the current frame with the frozen streaming backbone to obtain candidate state $\tilde{\mathbf{s}}_t$.
    \State Compute temporal drift from $\tilde{\mathbf{s}}_t$ and $\tilde{\mathbf{s}}_{t-1}$; update the EMA drift baseline.
    \State Estimate per-token process noise $\mathbf{q}_t$ with Eq.~\eqref{eq:q_gate_main}.
    \State Predict covariance: $\mathbf{p}_t^- = \mathbf{p}_{t-1} + \mathbf{q}_t$.
    \State Compute Kalman-style gain: $\mathbf{k}_t = \mathbf{p}_t^- / (\mathbf{p}_t^- + r)$.
    \State Update latent state: $\mathbf{s}_t = (1-\mathbf{k}_t)\odot\mathbf{s}_{t-1} + \mathbf{k}_t \odot \tilde{\mathbf{s}}_t$.
    \State Update covariance with the scalar Joseph recursion (Eq.~\ref{eq:joseph}).
    \State Store $\tilde{\mathbf{s}}_t$ and $\mathbf{p}_t$ for the next frame.
\end{algorithmic}
\end{algorithm}

We can distinguish TTT3R and FILT3R more clearly.
Both methods are adaptive in the broad sense that their update coefficients can change with time.
The difference is \emph{how} that change is generated.
TTT3R changes its gate from current cues, but it does not propagate a covariance state whose contraction makes the gate depend on accumulated confidence.
FILT3R does.
Its gain at time $t$ depends both on the current process-noise estimate $q_{t,i}$ and on the recursively propagated uncertainty $p_{t-1,i}$.
Therefore the same amount of instantaneous drift can lead to a smaller update after a long stable stretch than it would immediately after a transition.

The following proposition isolates this mechanism in the cleanest possible setting.
It is intentionally idealized: the scene is static, $q=0$, and no implementation clamps are applied.

\begin{proposition}[Static scenes yield naturally shrinking gains]
\label{prop:static_gain}
Consider a scalar token in a static regime: the true latent is constant $s^\star$, measurements are $\tilde{s}_t = s^\star + v_t$ with $v_t \sim \mathcal{N}(0,r)$, and process noise is zero ($q=0$).
Then the Kalman recursion gives
\begin{equation}
p_t = \frac{1}{\frac{1}{p_0} + \frac{t}{r}} = \mathcal{O}\!\left(\frac{1}{t}\right),
\qquad
k_t = \frac{p_{t-1}}{p_{t-1}+r} = \frac{1}{t + r/p_0}.
\label{eq:static_gain_decay}
\end{equation}
\end{proposition}

\begin{proof}
With $q=0$, the prediction step is $p_t^- = p_{t-1}$ and
$k_t = p_t^- / (p_t^- + r)$.
The Joseph update (Eq.~\ref{eq:joseph}) gives
\begin{align}
p_t
&= (1-k_t)^2 \,p_t^- + k_t^2\, r
= \frac{r^2}{(p_t^- + r)^2}\, p_t^-
  + \frac{(p_t^-)^2}{(p_t^- + r)^2}\, r \notag\\
&= \frac{p_t^-\, r\,(r + p_t^-)}{(p_t^- + r)^2}
= \frac{p_t^-\, r}{p_t^- + r}
= (1-k_t)\,p_t^-.
\label{eq:joseph_simplifies}
\end{align}
Therefore $1/p_t = 1/p_{t-1} + 1/r$, which telescopes to
$1/p_t = 1/p_0 + t/r$.
Substituting into $k_t = p_{t-1}/(p_{t-1}+r)$ yields Eq.~\eqref{eq:static_gain_decay}.
\end{proof}

Proposition~\ref{prop:static_gain} shows that, under repeated consistent observations, precision accumulates linearly, the posterior variance shrinks as $\mathcal{O}(1/t)$, and the gain shrinks at the same order.
The filter therefore behaves like an average over an ever longer history rather than an update with a fixed forgetting rate.
Overwrite (CUT3R, $k_t\equiv 1$) lacks this mechanism, and an instantaneous gate such as TTT3R may vary from frame to frame but has no propagated covariance recursion that forces this specific gain decay as confidence accumulates.

The next proposition considers a different regime, constant $q>0$.
Unlike Proposition~\ref{prop:static_gain}, which isolates the idealized zero-process-noise mechanism, this result characterizes the finite steady-state operating point of the same scalar recursion.
In Sec.~\ref{sec:gain_floor} we use it to interpret the implemented gain floor when $q_{t,i}$ settles near $q_{\min}$.

\begin{proposition}[Steady-state gain vs.\ $q/r$]
\label{prop:steady_state}
For the scalar random-walk model with constant process noise $q>0$ and measurement noise $r>0$,
the posterior variance converges to a unique positive steady state
\begin{equation}
p^\star = \frac{\sqrt{q^2+4qr}-q}{2},
\label{eq:pstar}
\end{equation}
and the corresponding steady-state Kalman-style gain is
\begin{equation}
k^\star = \frac{p^\star+q}{p^\star+q+r}
= \frac{\sqrt{q^2+4qr}+q}{\sqrt{q^2+4qr}+q+2r}.
\label{eq:steady_gain}
\end{equation}
\end{proposition}

\begin{proof}
At steady state the posterior variance satisfies $p_t = p_{t-1} = p^\star$.
We derive the fixed-point equation in three steps.

\emph{Step~1: Predicted variance.}
The prediction step gives $p^- = p^\star + q$.

\emph{Step~2: Steady-state gain.}
The Kalman-style gain is
$k^\star = p^-/(p^- + r) = (p^\star + q)/(p^\star + q + r)$.

\emph{Step~3: Fixed-point equation.}
For the Kalman-style gain above, the Joseph form simplifies to $(1-k^\star)\,p^-$ (as shown in Eq.~\ref{eq:joseph_simplifies}), giving the fixed-point condition
\[
p^\star = (1-k^\star)\,p^-
= \frac{r}{p^- + r}\cdot p^-
= \frac{r\,(p^\star + q)}{p^\star + q + r}.
\]
By cross-multiplying and expanding the left side, the $r\,p^\star$ terms cancel, yielding the quadratic
\begin{equation}
(p^\star)^2 + q\,p^\star - qr = 0.
\label{eq:pstar_quadratic}
\end{equation}
The quadratic formula gives Eq.~\eqref{eq:pstar}.
Substituting $p^\star$ back into $k^\star = (p^\star+q)/(p^\star+q+r)$ and noting that $p^\star + q = (q + \sqrt{q^2+4qr})/2$ yields Eq.~\eqref{eq:steady_gain}.
Finally, to show the convergence, write the scalar variance recursion as
\[
p_t = f(p_{t-1}), \qquad
f(p)=\frac{r(p+q)}{p+q+r}.
\]
For all \(p\ge 0\),
\[
f'(p)=\frac{r^2}{(p+q+r)^2}<1,
\]
so \(f\) is a contraction on \([0,\infty)\).
Therefore the positive fixed point is unique and the recursion converges to it.
\end{proof}

Proposition~\ref{prop:steady_state} gives the right interpretation of the adaptive process-noise term:
with fixed $r$, the ratio $q/r$ determines the operating point on the gain curve.
A fixed gate chooses one point on this curve.
TTT3R can still change its gate heuristically from frame to frame, but it does so without a propagated covariance state and therefore without the same confidence-accumulation mechanism.
FILT3R combines both ingredients: $q_t$ provides frame-dependent adaptivity, while $p_t$ carries accumulated confidence across time.

\begin{figure}[!tb]
\centering
\includegraphics[width=0.75\linewidth]{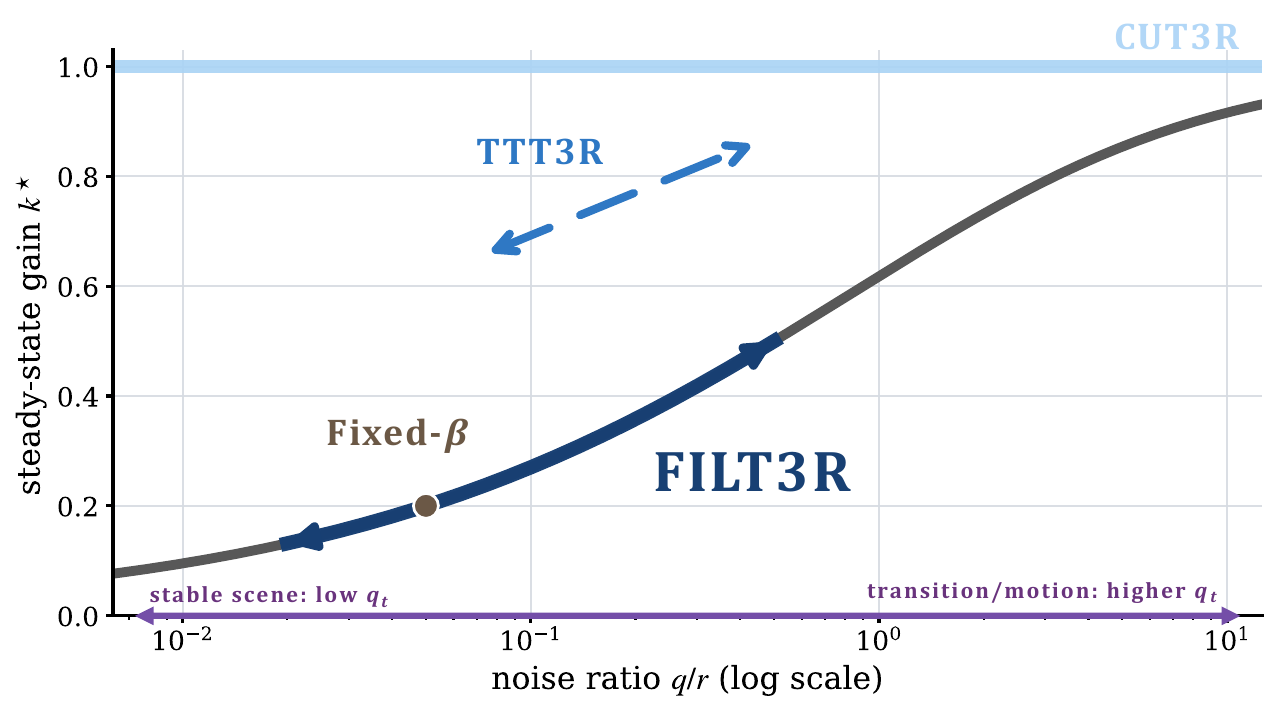}
\caption{\textbf{How different update rules relate to the gain curve.} With fixed $r$, the steady-state Kalman-style gain is determined by the noise ratio ($q/r$). Overwrite (CUT3R) corresponds to the limit ($k\to1$). A fixed gate chooses one constant operating point. FILT3R adapts $q_t$ and therefore moves along the curve, becoming conservative in stable regimes and reopening at transitions. TTT3R is also adaptive, but its gate is heuristic and based on current cues rather than propagated covariance, so it is shown schematically off the steady-state curve.
}
\label{fig:kstar_curve}
\end{figure}

\subsection{Implemented gain floor}
\label{sec:gain_floor}

If a token remains in a stable regime so that $q_{t,i}\approx q_{\min}$, Proposition~\ref{prop:steady_state} implies the unclamped steady-state gain
\begin{equation}
k_{\mathrm{floor}}^\star
=
\frac{\sqrt{q_{\min}^2+4q_{\min}r}+q_{\min}}
{\sqrt{q_{\min}^2+4q_{\min}r}+q_{\min}+2r}.
\label{eq:gain_floor}
\end{equation}
For the actual hyperparameters ($q_{\min}=0.02$, $r=1.0$), Eq.~\eqref{eq:gain_floor} gives $k_{\mathrm{floor}}^\star\approx 0.132$ before applying the clamp.
The practical interpretation is therefore not that FILT3R freezes completely, but that it settles into a substantially more conservative update regime while preserving responsiveness to scene change.

This distinction matters for interpretation.
The idealized $\mathcal{O}(1/t)$ result explains why recursive uncertainty propagation expands the effective memory horizon relative to overwrite or fixed-gain updates.
The implemented model inherits the same qualitative mechanism, but its asymptotic behavior is controlled by $q_{\min}$ and $[k_{\min},k_{\max}]$.
Figure~\ref{fig:gain_floor_diag} shows that on long sequences, FILT3R's gain approaches the predicted floor, confirming that the update has converged to its stable-regime operating point.

\begin{figure}[!tb]
\centering
\includegraphics[width=0.9\linewidth]{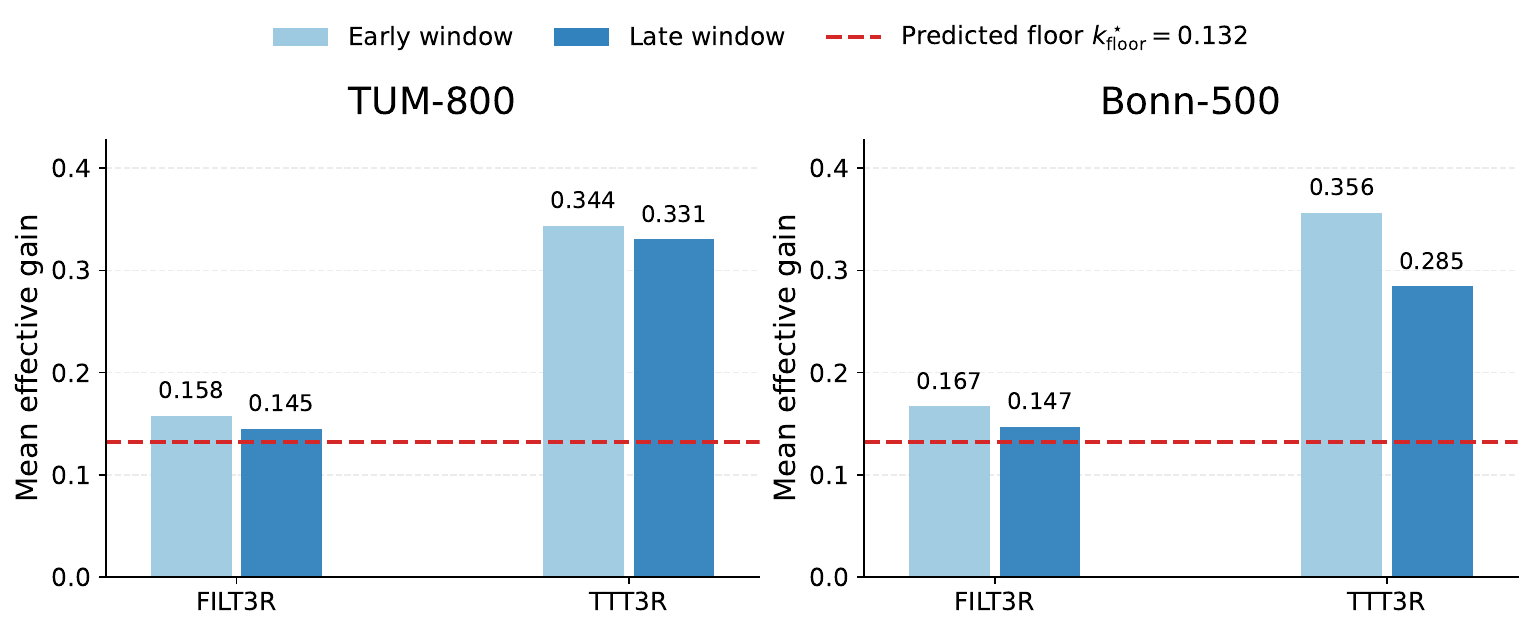}
\caption{\textbf{The deployed filter settles near a finite gain floor rather than freezing completely.}
The dashed line marks the implementation-faithful prediction $k_{\mathrm{floor}}^\star \approx 0.132$ from Eq.~\eqref{eq:gain_floor}.
Here, the early and late windows denote the first and last 20\% of frames of each analyzed sequence.
Across both TUM-800 and Bonn-500, FILT3R's late-window mean gain remains only $0.013$--$0.015$ above this floor, while TTT3R stays substantially more reactive.
This is consistent with FILT3R entering a conservative stable regime and reopening only when transition frames push $q_t$ upward.}
\label{fig:gain_floor_diag}
\end{figure}

\subsection{Distinguishing FILT3R from fixed-EMA smoothing}
\label{sec:not_ema}

The central difference between FILT3R and overwrite, heuristic gating, or fixed-EMA smoothing is that FILT3R \emph{propagates uncertainty}.
This propagation allows the update to become increasingly conservative as evidence accumulates, while still reopening when elevated process noise signals a genuine scene transition.
FILT3R does not merely select a different constant for the gate; it changes the \emph{dynamics} of the gate itself, making it time-dependent through propagated uncertainty.

A natural alternative hypothesis is that a globally conservative fixed EMA already explains the observed long-horizon gains.
The fixed-$\beta$ comparison in Table~\ref{tab:ablation} reveals a more nuanced picture: very small fixed $\beta$ can indeed suppress TUM-800 drift and even improve Bonn-500 Abs Rel, but at the cost of over-smoothing local motion.
For instance, reading from the extended sweep in Figure~\ref{fig:fixedbeta_long}, $\beta=0.01$ reduces TUM-800 $\text{ATE}_{\text{orig}}$ to 0.080 and Bonn-500 Abs Rel to 0.085, yet its TUM-800 RPE-r jumps to 1.91; even $\beta=0.03$ yields RPE-r of 1.25.
FILT3R therefore acts as a \emph{selective} smoother: it contracts updates as evidence accumulates, but does not commit to a single forgetting rate for the entire stream.

In our experiments, the fixed-$\beta$ baseline in Table~\ref{tab:ablation} uses the development-set choice $\beta=0.05$.
In the fixed-EMA development sweep, $\beta$ is selected by minimizing the sum of camera and depth metrics for TUM-200 and Bonn-150.
The longer-horizon test benchmarks in Figure~\ref{fig:fixedbeta_long} were not used in that choice.

\begin{table}[!tb]
\centering
\caption{\textbf{FILT3R versus fixed smoothing: ablations on TUM-800 and Bonn-500.}
Removing adaptive process noise or variance propagation weakens long-horizon stability.
A tuned fixed-$\beta$ EMA can suppress certain global drift metrics, but does so by uniformly damping all updates, which substantially degrades local rotation accuracy.
No single constant coefficient matches FILT3R's joint performance across long-horizon drift, local motion accuracy, and depth quality.}
\label{tab:ablation_appendix}
\small
\setlength{\tabcolsep}{4pt}
\resizebox{\linewidth}{!}{%
\begin{tabular}{l|cccc|ccc}
\toprule
\multirow{2}{*}{\textbf{Variant}} &
\multicolumn{4}{c|}{\textbf{TUM-800 (pose)}} &
\multicolumn{3}{c}{\textbf{Bonn-500 (metric depth)}}\\
\cmidrule(lr){2-5}\cmidrule(lr){6-8}
& ATE$\downarrow$ & $\text{ATE}_{\text{orig}}\downarrow$ & RPE-t$\downarrow$ & RPE-r$\downarrow$
& AbsRel$\downarrow$ & $\delta{<}1.25\uparrow$ & logRMSE$\downarrow$ \\
\midrule
FILT3R (full)                            & 0.057 & 0.107 & 0.009 & 0.362 & 0.089 & 94.4 & 0.148 \\
Fixed-$\mathbf{Q}$ (fixed $r$)           & 0.100 & 0.229 & 0.009 & 0.393 & 0.099 & 91.9 & 0.155 \\
No variance propagation (reset-$\mathbf{P}$) & 0.149 & 0.489 & 0.009 & 0.446 & 0.099 & 91.1 & 0.154 \\
Fixed-$\beta$ EMA ($\beta{=}0.05$, tuned) & 0.049 & 0.078 & 0.010 & 0.658 & 0.093 & 93.5 & 0.152 \\
No EMA normalization for drift           & 0.058 & 0.109 & 0.010 & 0.367 & 0.091 & 94.1 & 0.150 \\
Adaptive measurement noise ($r_t$)       & 0.074 & 0.148 & 0.009 & 0.372 & 0.093 & 93.9 & 0.151 \\
\bottomrule
\end{tabular}}
\end{table}

\paragraph{Exact definitions of the ablations.}
\emph{Fixed-$\mathbf{Q}$} replaces Eq.~\eqref{eq:q_gate_main} with a stream-independent constant $q_{t,i}\equiv\bar q$ while keeping variance propagation and fixed $r$.
The ablation defaults to $\bar q=\tfrac{1}{2}(q_{\min}+q_{\max})=0.26$ ($q_{\min}=0.02$, $q_{\max}=0.5$).
\emph{No variance propagation (reset-$\mathbf{P}$)} resets $p_{t-1,i}$ to the initialization value before each update, so the gain uses the same functional form but cannot accumulate confidence across frames.
\emph{Fixed-$\beta$ EMA} removes the Kalman variance recursion and sets $k_{t,i}\equiv\beta$ for all tokens and frames.
\emph{No EMA normalization} feeds raw drift magnitudes into the same sigmoid gate instead of the normalized score $g_{t,i}$.
\emph{Adaptive measurement noise} replaces the scalar $r$ with token-wise $r_{t,i}$ estimated from normalized attention entropy while leaving the rest of the update unchanged.
For the reported ablation, let $a_{t,i,j}$ denote the absolute aggregated cross-attention from state token $i$ to image token $j$ in the first cross-attention layer, let
\[
p_{t,i,j}=\frac{a_{t,i,j}}{\sum_j a_{t,i,j}+\varepsilon},
\qquad
H_{t,i}=-\frac{\sum_j p_{t,i,j}\log p_{t,i,j}}{\log K+\varepsilon},
\]
where $K$ is the number of image tokens, and let $\hat{H}_t$ be the EMA of the sequence-mean entropy with update rate $0.05$.
The ablation then uses
\[
r_{t,i}=r_{\min}+r_{\mathrm{scale}}\,
\sigma\!\Big(\alpha_r\big(H_{t,i}/\hat{H}_t-\tau_r\big)\Big),
\]
with $r_{\min}=1.0$, $r_{\mathrm{scale}}=1.0$, $\alpha_r=8.0$, and $\tau_r=1.0$.
The entropy signal is taken from the first cross-attention layer in all reported adaptive-$r_t$ runs.

\paragraph{Diagnostic conventions used in Figures~\ref{fig:gain_floor_diag} and~\ref{fig:transition_timeline}.}
The gain-floor summary in Figure~\ref{fig:gain_floor_diag} uses the first and last 20\% of frames as the early and late windows.
For the transition timelines in Figure~\ref{fig:transition_timeline}, the \emph{transition score} is the per-frame mean normalized drift
$\bar g_t=\tfrac{1}{N}\sum_i g_{t,i}$ from Eq.~\eqref{eq:q_gate_main}.
The plotted score is an MA-11 curve, meaning a trailing moving average over the previous 11 frames (including frame $t$).
The \emph{transition threshold} is the 90th percentile of this smoothed FILT3R score, and detected transition windows are contiguous regions above that threshold.
\emph{Effective gain} denotes the per-frame mean of the applied per-token gain $k_{t,i}$ after clamping, \emph{posterior covariance} denotes the per-frame mean of the post-update variance $p_{t,i}$, and the \emph{update ratio} is
\[
\rho_t=
\frac{\frac{1}{N}\sum_i \lVert s_{t,i}-s_{t-1,i}\rVert_2}
{\frac{1}{N}\sum_i \lVert \tilde{s}_{t,i}-s_{t-1,i}\rVert_2},
\]
that is, the mean applied update magnitude divided by the mean candidate-update magnitude.

\begin{figure}[!t]
\centering
\begin{minipage}[t]{0.74\linewidth}
\centering
\includegraphics[width=\linewidth]{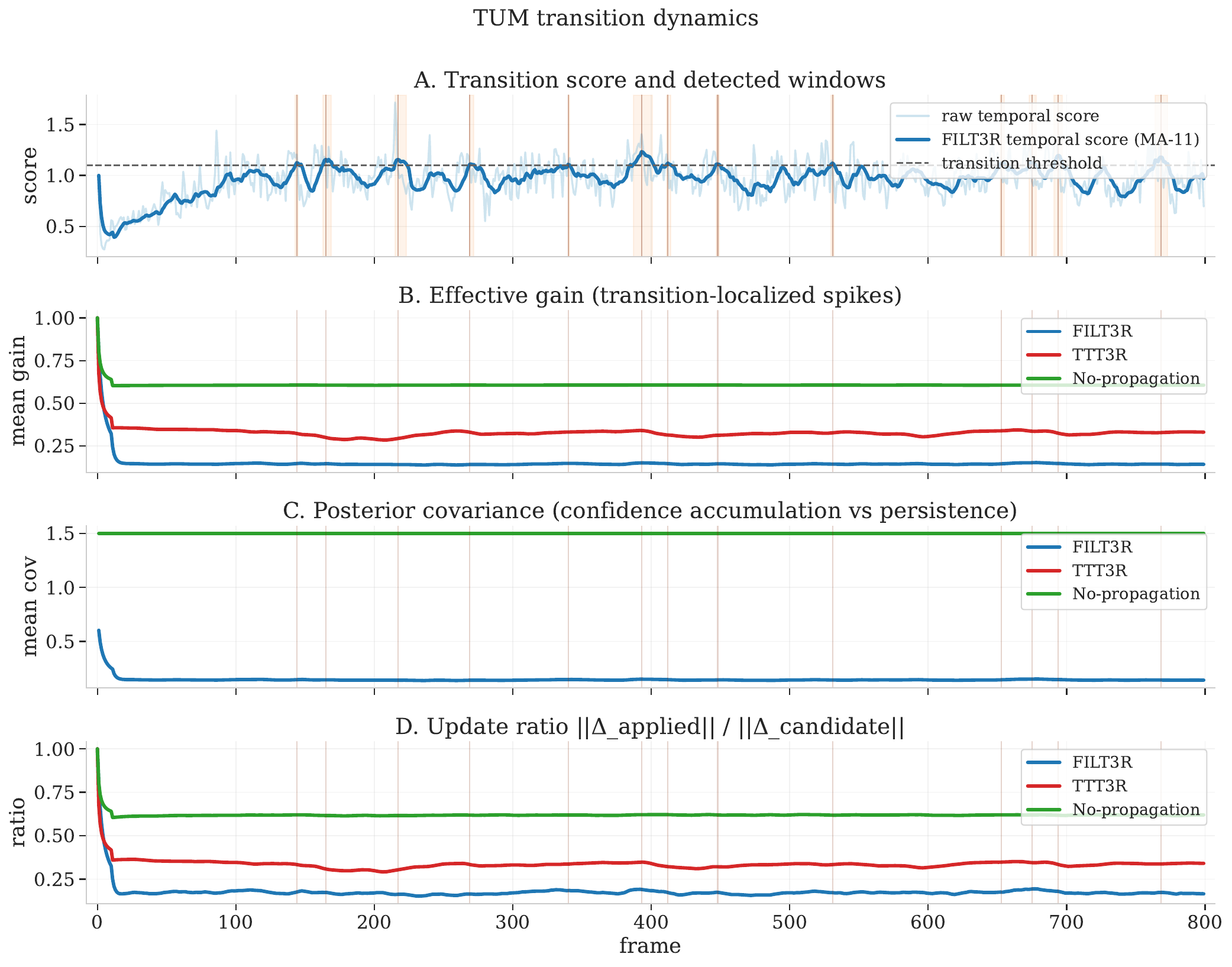}
\caption*{(a) TUM-800 transition timeline}
\end{minipage}

\vspace{0pt}
\begin{minipage}[t]{0.74\linewidth}
\centering
\includegraphics[width=\linewidth]{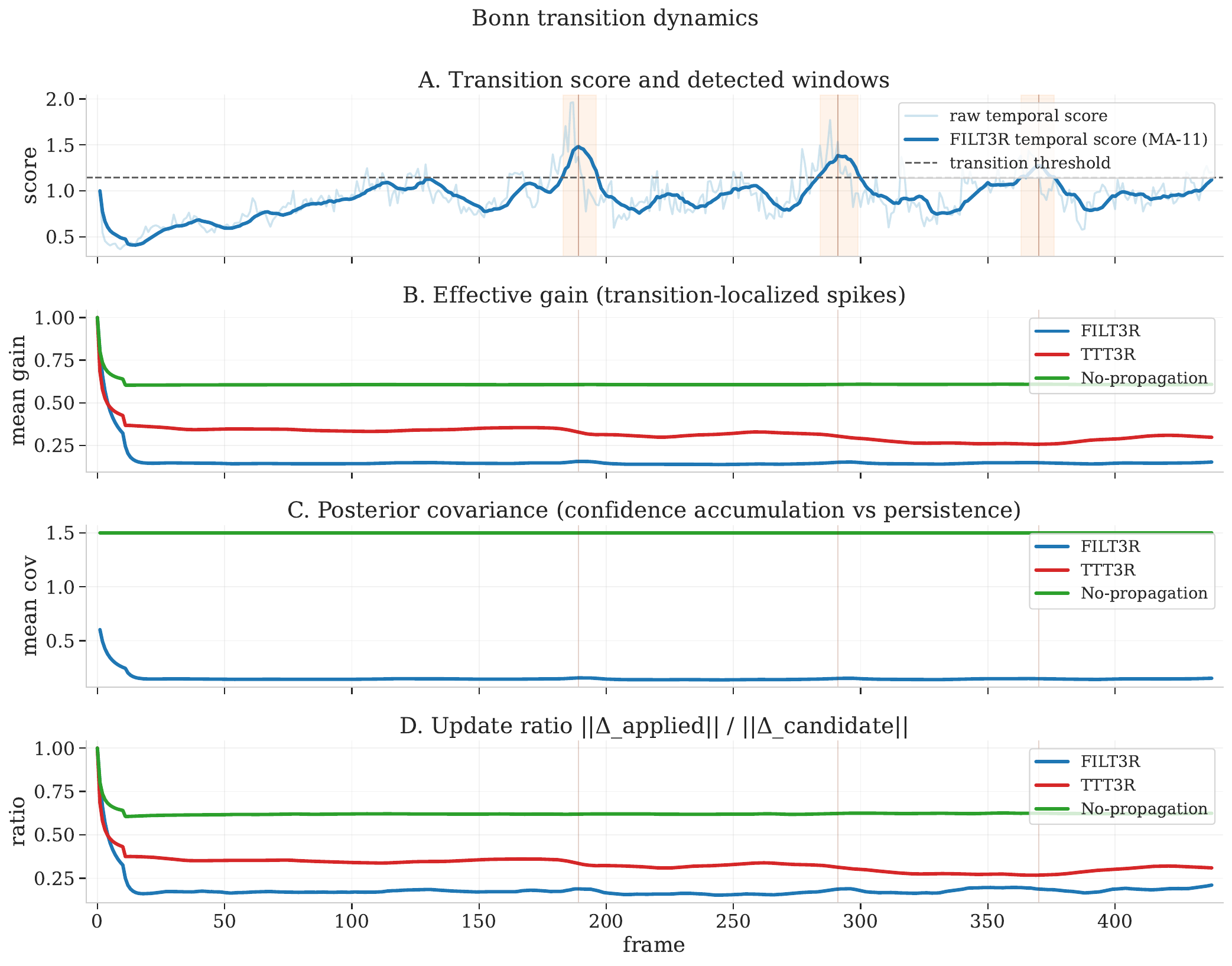}
\caption*{(b) Bonn-500 transition timeline}
\end{minipage}
\caption{\textbf{FILT3R contracts updates in stable regimes and reopens at transitions.}
Panel A plots the per-frame mean temporal score $\bar g_t$ together with its MA-11 smoothing, the 90th-percentile transition threshold, and the detected above-threshold windows.
Panels B--D report the corresponding mean effective gain, mean posterior covariance, and update ratio $\rho_t$.
Detected transition windows raise process noise, while propagated covariance keeps FILT3R in a substantially lower-update regime than both TTT3R and the no-propagation ablation.
This behavior provides the mechanistic explanation for the slower error growth observed in Figure~\ref{fig:long_horizon_metrics}.}
\label{fig:transition_timeline}
\end{figure}

\section{Why periodic hard resets are not sufficient}
\label{app:reset_analysis}
An alternative strategy for managing long-horizon drift is to periodically reset the recurrent state.
We compare against the periodic hard-reset baseline used in the official TTT3R evaluation setup: the recurrent state is reinitialized every 100 frames while the frozen backbone and evaluator settings remain unchanged.
At each reset, the persistent state and local memory return to their start-of-sequence values and the temporal buffers are cleared.
While resets can reduce within-segment drift, they also discard accumulated scene context and risk introducing scale discontinuities at reset boundaries.
The relevant criterion is whether resets preserve a coherent belief state over the full sequence, not merely whether they reduce within-segment drift.

\begin{figure}[!tb]
\centering
\includegraphics[width=\linewidth]{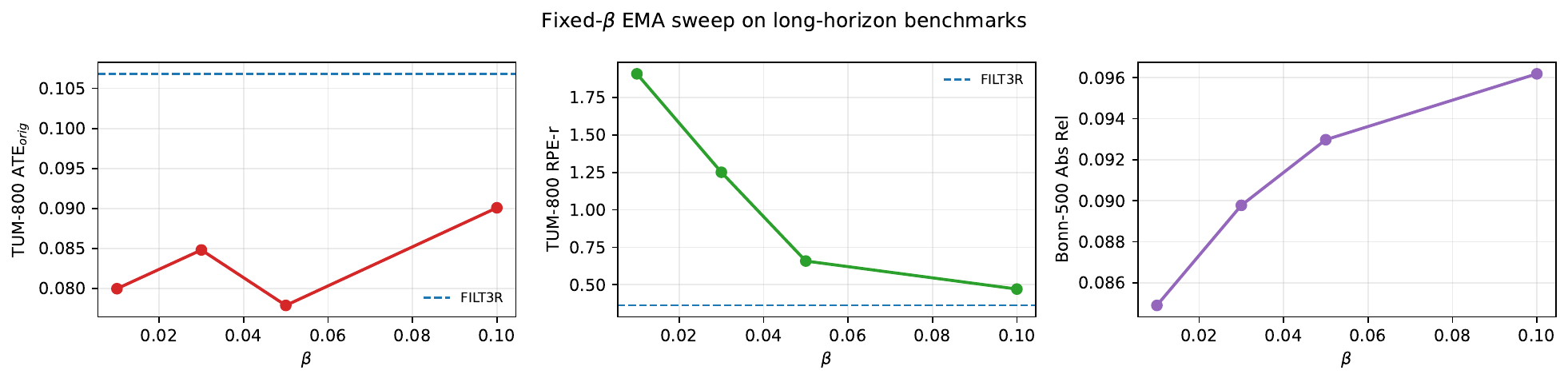}
\caption{\textbf{Extended long-horizon fixed-$\beta$ sweep.}
Very small fixed EMA coefficients can outperform FILT3R on individual long-horizon metrics, but only by globally over-smoothing the stream: the same settings that reduce $\text{ATE}_{\text{orig}}$ or improve depth also produce substantially elevated TUM-800 RPE-r.
FILT3R occupies a stronger operating region because it does not commit to a single forgetting rate for the entire rollout.}
\label{fig:fixedbeta_long}
\end{figure}

\begin{figure}[!tb]
\centering
\includegraphics[width=\linewidth]{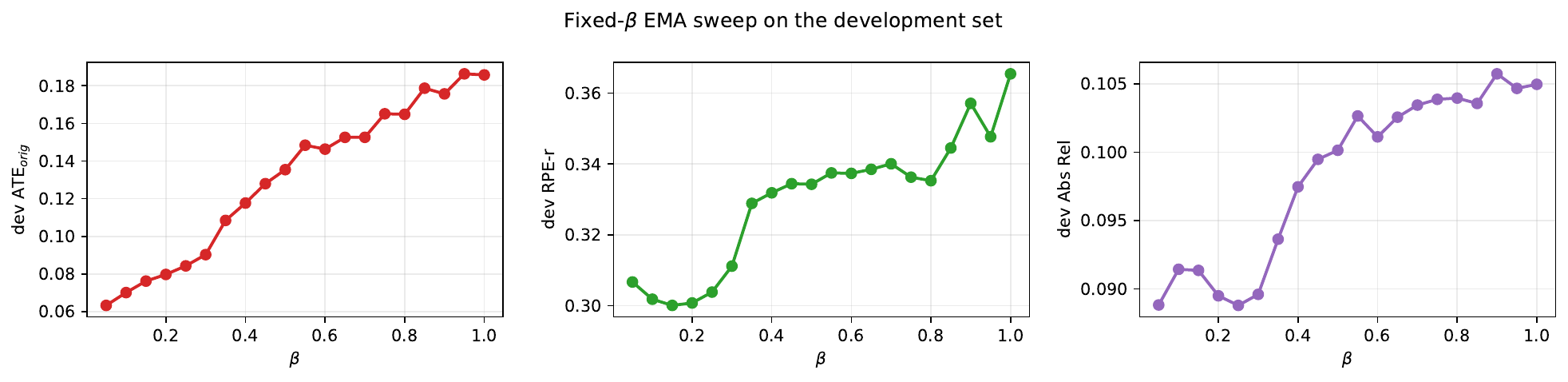}
\caption{\textbf{Development-set fixed-$\beta$ sweep.}
This sweep provides context for representative fixed-EMA operating points and shows the development-set optimum at $\beta=0.05$ under the explicit objective above.
The more informative comparison is the long-horizon sweep above: a fixed gate can be tuned toward one regime, yet it remains globally conservative rather than transition-aware.}
\label{fig:fixedbeta_dev}
\end{figure}

Tables~\ref{tab:pose_long_reset} and~\ref{tab:depth_long_reset} address this question on TUM-RGBD (camera pose) and Bonn (depth), respectively.
On TUM-RGBD (Table~\ref{tab:pose_long_reset}), TTT3R+Reset does improve locally aligned metrics such as ATE and RPE at moderate lengths, yet FILT3R remains substantially better on origin-aligned trajectory error at all evaluated horizons---for example, $\text{ATE}_{\text{orig}}$ at TUM-800 is 0.107 for FILT3R versus 0.221 for TTT3R+Reset.
On Bonn (Table~\ref{tab:depth_long_reset}), per-sequence scale alignment masks the issue: resets appear helpful at shorter horizons.
Under metric-scale evaluation, however, FILT3R consistently outperforms the reset baseline as the horizon grows, because it avoids the scale discontinuities that resets introduce at segment boundaries.
Figure~\ref{fig:reset_sweeps} visualizes these trends as a function of sequence length, and Figure~\ref{fig:reset_failures} shows representative qualitative failure cases where resets cause visible scale jumps.

\begin{table*}[!tb]
\centering
\caption{\textbf{Resets improve within-segment metrics but fail to reduce cumulative drift.}
On TUM-RGBD, TTT3R+Reset improves certain locally aligned metrics, yet FILT3R remains substantially better on origin-aligned trajectory error, which directly exposes accumulated drift.}
\label{tab:pose_long_reset}
\resizebox{\textwidth}{!}{%
\begin{tabular}{l | cccc | cccc | cccc}
\toprule
\multirow{2}{*}{\textbf{Method}} & \multicolumn{4}{c|}{\textbf{TUM-400}} & \multicolumn{4}{c|}{\textbf{TUM-600}} & \multicolumn{4}{c}{\textbf{TUM-800}}\\
\cmidrule(lr){2-5}\cmidrule(lr){6-9}\cmidrule(lr){10-13}
 & ATE$\downarrow$ & $\text{ATE}_{\text{orig}}\!\downarrow$ & RPE-t$\downarrow$ & RPE-r$\downarrow$
 & ATE$\downarrow$ & $\text{ATE}_{\text{orig}}\!\downarrow$ & RPE-t$\downarrow$ & RPE-r$\downarrow$
 & ATE$\downarrow$ & $\text{ATE}_{\text{orig}}\!\downarrow$ & RPE-t$\downarrow$ & RPE-r$\downarrow$ \\
\midrule
CUT3R         & 0.109 & 0.302 & 0.010 & 0.381 & 0.145 & 0.425 & 0.008 & 0.430 & 0.173 & 0.493 & 0.008 & 0.486 \\
Point3R       & 0.119 & 0.179 & 0.028 & 0.640 & 0.133 & 0.205 & 0.026 & 0.912 & 0.157 & 0.241 & 0.024 & 1.127 \\
TTT3R         & 0.055 & 0.128 & 0.010 & 0.328 & 0.084 & 0.176 & 0.009 & 0.365 & 0.097 & 0.214 & 0.009 & 0.387 \\
TTT3R+Reset   & 0.047 & 0.140 & \textbf{0.008} & 0.308 & 0.059 & 0.181 & \textbf{0.007} & \textbf{0.331} & 0.077 & 0.221 & \textbf{0.007} & \textbf{0.338} \\
\oursrow
\oursname     & \textbf{0.033} & \textbf{0.074} & 0.009 & \textbf{0.305} & \textbf{0.042} & \textbf{0.082} & 0.009 & 0.332 & \textbf{0.057} & \textbf{0.107} & 0.009 & 0.362 \\
\bottomrule
\end{tabular}}
\end{table*}

\begin{table*}[!tb]
\centering
\caption{\textbf{Reset depth comparison on Bonn.}
(a) Per-sequence scale alignment at 300, 400, and 500 frames shows that resets can help transiently.
(b) Under metric-scale evaluation, however, FILT3R remains consistently stronger as the horizon grows, because it avoids the scale inconsistencies introduced at reset boundaries.}
\label{tab:depth_long_reset}

{\small (a) \textbf{Per-sequence scale} (Bonn-300 / Bonn-400 / Bonn-500)}
\vspace{1pt}

\resizebox{\textwidth}{!}{%
\begin{tabular}{l | ccc | ccc | ccc}
\toprule
\multirow{2}{*}{\textbf{Method}} &
\multicolumn{3}{c|}{\textbf{Bonn-300}} &
\multicolumn{3}{c|}{\textbf{Bonn-400}} &
\multicolumn{3}{c}{\textbf{Bonn-500}}\\
\cmidrule(lr){2-4}\cmidrule(lr){5-7}\cmidrule(lr){8-10}
& Abs Rel$\downarrow$ & $\delta{<}1.25\uparrow$ & log RMSE$\downarrow$ & Abs Rel$\downarrow$ & $\delta{<}1.25\uparrow$ & log RMSE$\downarrow$ & Abs Rel$\downarrow$ & $\delta{<}1.25\uparrow$ & log RMSE$\downarrow$ \\
\midrule
CUT3R       & 0.090 & 93.8 & 0.150 & 0.091 & 93.5 & 0.150 & 0.086 & 94.0 & 0.145 \\
Point3R     & 0.083 & 94.7 & 0.149 & 0.083 & 94.6 & 0.147 & 0.081 & 94.8 & 0.145 \\
TTT3R       & \textbf{0.079} & 95.0 & \textbf{0.142} & 0.078 & 95.1 & \textbf{0.139} & \textbf{0.075} & 95.3 & \textbf{0.136} \\
TTT3R+Reset & 0.085 & 94.3 & 0.147 & 0.082 & 94.5 & 0.142 & 0.081 & 94.8 & 0.141 \\
\oursrow
\oursname   & 0.080 & \textbf{95.1} & 0.143 & \textbf{0.078} & \textbf{95.3} & 0.140 & 0.076 & \textbf{95.4} & 0.138 \\
\bottomrule
\end{tabular}}

\vspace{4pt}
{\small (b) \textbf{Metric scale} (Bonn-300 / Bonn-400 / Bonn-500)}
\vspace{1pt}

\resizebox{\textwidth}{!}{%
\begin{tabular}{l | ccc | ccc | ccc}
\toprule
\multirow{2}{*}{\textbf{Method}} &
\multicolumn{3}{c|}{\textbf{Bonn-300}} &
\multicolumn{3}{c|}{\textbf{Bonn-400}} &
\multicolumn{3}{c}{\textbf{Bonn-500}}\\
\cmidrule(lr){2-4}\cmidrule(lr){5-7}\cmidrule(lr){8-10}
& Abs Rel$\downarrow$ & $\delta{<}1.25\uparrow$ & log RMSE$\downarrow$ & Abs Rel$\downarrow$ & $\delta{<}1.25\uparrow$ & log RMSE$\downarrow$ & Abs Rel$\downarrow$ & $\delta{<}1.25\uparrow$ & log RMSE$\downarrow$ \\
\midrule
CUT3R       & 0.107 & 88.7 & 0.162 & 0.107 & 89.6 & 0.161 & 0.101 & 90.6 & 0.156 \\
Point3R     & 0.095 & \textbf{94.4} & 0.155 & 0.092 & \textbf{94.3} & 0.152 & 0.092 & 94.4 & 0.151 \\
TTT3R       & 0.108 & 90.1 & 0.163 & 0.104 & 91.2 & 0.158 & 0.100 & 92.1 & 0.154 \\
TTT3R+Reset & 0.097 & 92.2 & 0.157 & 0.094 & 92.9 & 0.151 & 0.094 & 93.5 & 0.151 \\
\oursrow
\oursname   & \textbf{0.093} & 93.8 & \textbf{0.152} & \textbf{0.090} & 94.1 & \textbf{0.149} & \textbf{0.089} & \textbf{94.4} & \textbf{0.148} \\
\bottomrule
\end{tabular}}
\end{table*}

\section{Evaluation protocol and reproducibility}
\label{sec:appendix_protocol}
\paragraph{Common setup.}
Unless stated otherwise, CUT3R, TTT3R, TTT3R+Reset, and FILT3R share the same frozen CUT3R checkpoint as in the main paper; only the online update rule differs.
Point3R results are reproduced with the official codebase.
Unless stated otherwise, every quantitative result in this appendix uses the main FILT3R model with a single set of hyperparameters across all tasks.
For each dataset, we extract frames at the native frame rate and truncate sequences to fixed prefix lengths; we refer to these as \emph{prepared} sequences throughout.

\begin{figure}[!tb]
\centering
\begin{minipage}[t]{0.48\linewidth}
\centering
\includegraphics[width=0.8\linewidth]{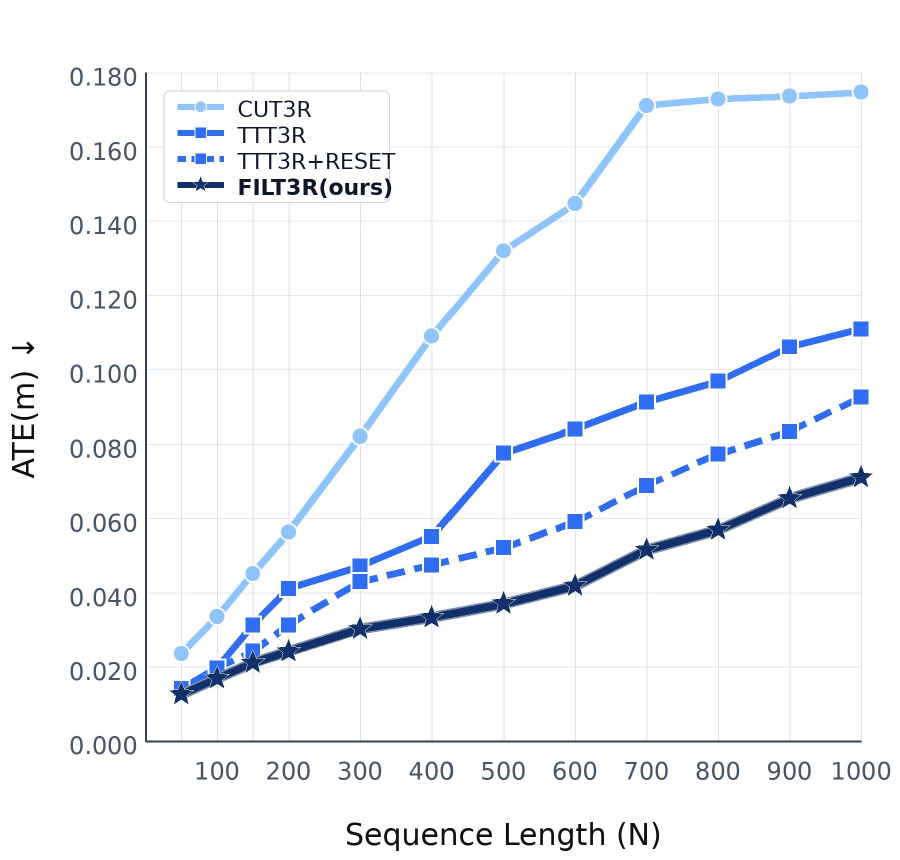}
\caption*{(a) TUM-RGBD: ATE with reset}
\end{minipage}\hfill
\begin{minipage}[t]{0.48\linewidth}
\centering
\includegraphics[width=0.8\linewidth]{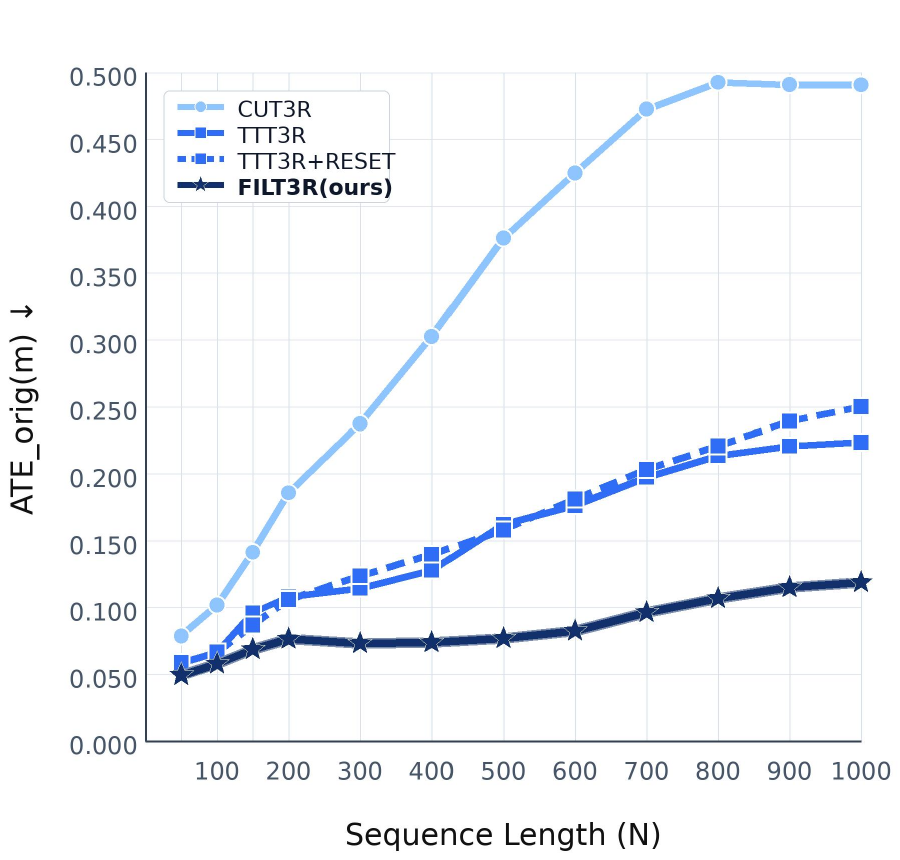}
\caption*{(b) TUM-RGBD: $\text{ATE}_{\text{orig}}$ with reset}
\end{minipage}

\begin{minipage}[t]{0.48\linewidth}
\centering
\includegraphics[width=0.8\linewidth]{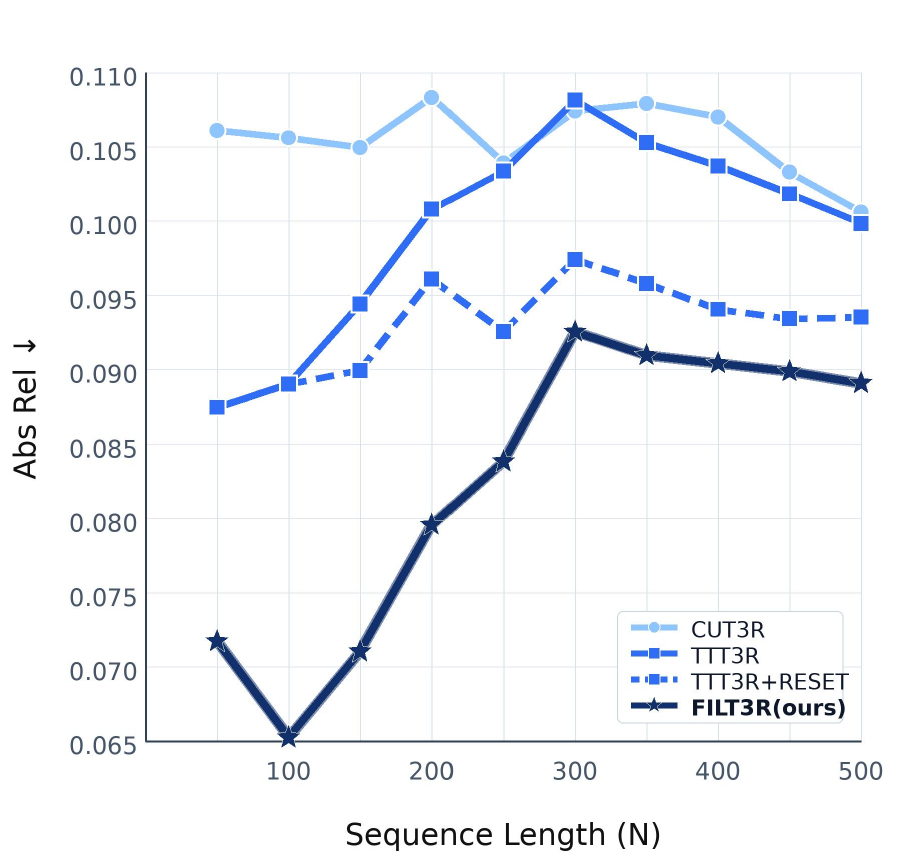}
\caption*{(c) Bonn: Abs Rel with reset}
\end{minipage}\hfill
\begin{minipage}[t]{0.48\linewidth}
\centering
\includegraphics[width=0.8\linewidth]{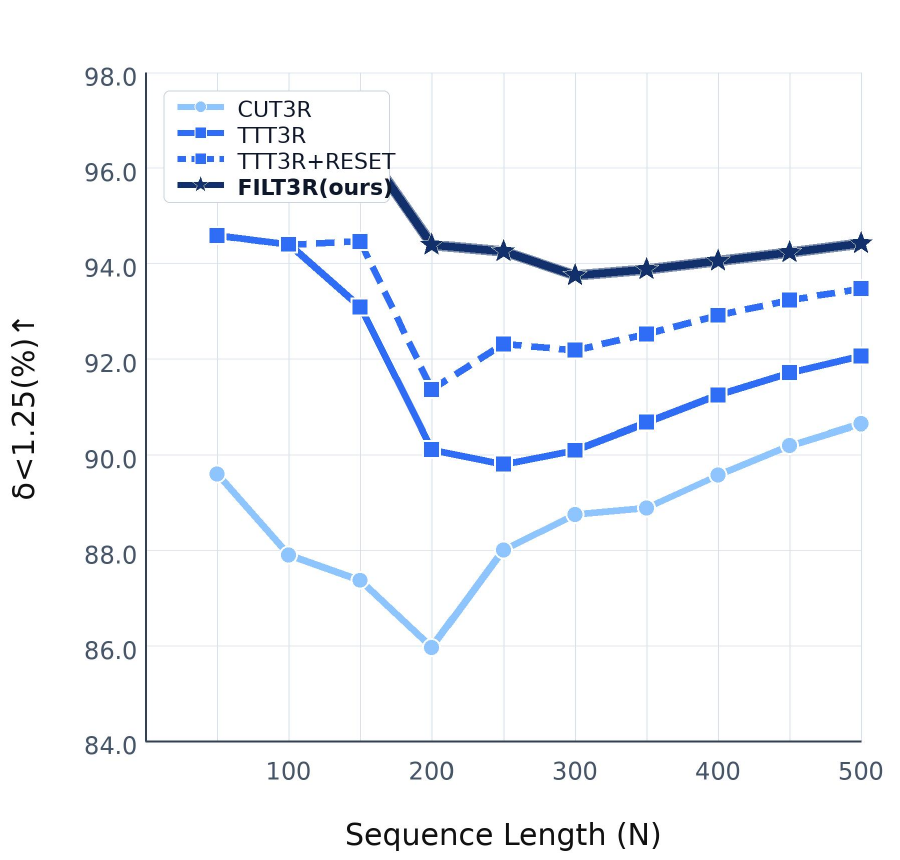}
\caption*{(d) Bonn: $\delta{<}1.25$ with reset}
\end{minipage}
\caption{\textbf{Resets improve short segments but not the persistent state.}
The reset baseline reduces locally aligned drift at moderate lengths, but its benefit diminishes as scale discontinuities accumulate across reset boundaries.
FILT3R instead maintains a single continuously filtered state throughout the rollout.}
\label{fig:reset_sweeps}
\end{figure}

\begin{figure}[!tb]
\centering
\includegraphics[width=0.98\linewidth]{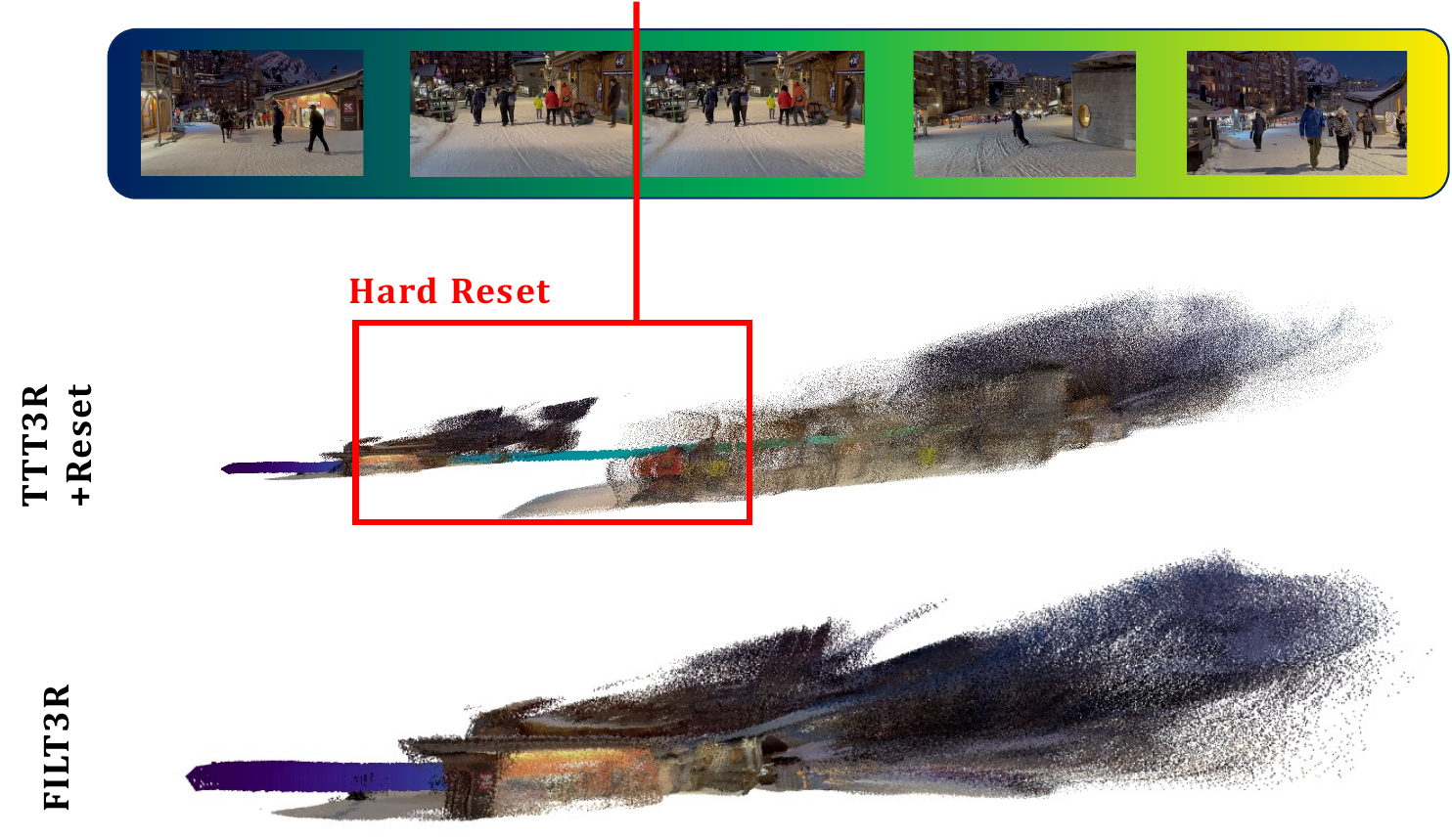}
\caption{\textbf{Representative reset failure cases.}
Resetting discards previously accumulated scene context, causing scale jumps near reset boundaries and inconsistent behavior on revisited regions.
FILT3R avoids these artifacts by maintaining a continuously filtered latent state.}
\label{fig:reset_failures}
\end{figure}

\paragraph{Training-bound reference and evaluator settings.}
The public backbone checkpoint is \texttt{cut3r\_512\_dpt\_4\_64.pth}.
This released CUT3R model is trained on \emph{4--64 views}, so Table~\ref{tab:protocol} reports evaluation lengths relative to the conservative 64-view upper training bound.
The same table also collects the evaluator-side settings that materially affect the reported metrics: alignment mode for pose, scale handling for depth, valid-depth masks, and reconstruction crop / ICP settings.

\begin{table*}[!tb]
\centering
\caption{\textbf{Exact evaluator settings and relation to the released checkpoint's training bound.}
The checkpoint tag and README specify a 4--64-view model; we therefore state how far each benchmark extends beyond the 64-view upper bound.}
\label{tab:protocol}
\small
\setlength{\tabcolsep}{4pt}
\resizebox{\textwidth}{!}{%
\begin{tabular}{l|c|c|c|p{7.9cm}}
\toprule
\textbf{Benchmark} & \textbf{Training ref.} & \textbf{Prefix lengths} & \textbf{Prefix / 64} & \textbf{Evaluator details} \\
\midrule
TUM-RGBD pose & CUT3R 4--64 views & 50--1000 & $0.78\times$--$15.63\times$ & ATE uses \texttt{align=True}, \texttt{correct\_scale=True}; $\text{ATE}_{\text{orig}}$ uses \texttt{align\_origin=True}; RPE is evaluated at $\Delta{=}1$ frame with similarity alignment. \\
Bonn depth & CUT3R 4--64 views & 50--500 & $0.78\times$--$7.81\times$ & Main long-horizon table uses metric scale; the reset comparison additionally reports per-sequence scale alignment. Valid pixels satisfy GT depth $>0$ and $<70$m; predictions are bicubically resized and averaged with valid-pixel weighting. \\
KITTI depth & CUT3R 4--64 views & 300 / 400 / 500 & $4.69\times$--$7.81\times$ & Both metric-scale and per-sequence-scale metrics are reported. The evaluator keeps the dataset's native depth range (\texttt{max\_depth=None}) and uses the same valid-pixel-weighted aggregation. \\
7-Scenes recon & CUT3R 4--64 views & 100 / 200 / 300 / 500 / 1000 & $1.56\times$--$15.63\times$ & Input size 512 with a 224-pixel center crop for point cloud optimization. With ICP applied to report Accuracy, Completeness, and Normal Consistency. \\
NRGBD recon & CUT3R 4--64 views & 300 / 400 / 500 & $4.69\times$--$7.81\times$ & The same reconstruction evaluator is used on the nine prepared NRGBD scenes.\\
\bottomrule
\end{tabular}}
\end{table*}

\paragraph{Initialization, reset, and efficiency-benchmark details.}
On the first valid frame of a stream, the state is updated by direct overwrite with the current candidate, the covariance is initialized uniformly to $p_0$, and the previous-candidate buffer is set to that same candidate.
The drift EMA is initialized only when the first temporal difference is available and is then clamped to $\Delta_{\mathrm{floor}}$.
For TTT3R+Reset, we follow the official 100-frame hard-reset setup: the persistent state and local memory are reinitialized to their start-of-sequence values and temporal buffers are cleared, with no interval tuning on the evaluation sets.
For the runtime / memory benchmark in the main paper, the protocol is NRGBD at 500 frames and $512\times384$ resolution, with 2 warmup runs and 10 timed runs, on a single-RTX 5090 32GB GPU.

\paragraph{Long-horizon camera pose on TUM-RGBD.}
Following~\cite{chen2026ttt3r}, we evaluate eight prepared TUM-dynamics sequences:
\begin{itemize}
    \item \texttt{freiburg3\_sitting\_\{halfsphere,rpy,static,xyz\}} \\
    \item \texttt{freiburg3\_walking\_\{halfsphere,rpy,static,xyz\}}.
\end{itemize}

For each prepared prefix length in \{50, 100, 150, 200, 300, 400, 500, 600, 700, 800, 900, 1000\}, we evaluate the full prefix and average metrics across the eight sequences.
ATE is computed with global similarity alignment via \texttt{evo} package;
$\text{ATE}_{\text{orig}}$ uses origin alignment only to expose accumulated drift; RPE is computed at $\Delta=1$ frame.

\paragraph{Long-horizon video depth on Bonn.}
Following~\cite{chen2026ttt3r}, we evaluate the five long Bonn sequences
\texttt{balloon2}, \texttt{crowd2}, \texttt{crowd3}, \texttt{person\_tracking2}, and \texttt{synchronous}.
Prepared prefix lengths are \{50, 100, 150, 200, 250, 300, 350, 400, 450, 500\}.
The main-paper long-horizon depth table uses metric-scale evaluation; the reset table additionally reports per-sequence scale alignment.
Depth metrics are aggregated with valid-pixel weighting.

\paragraph{Long-horizon video depth on KITTI.}
Following~\cite{chen2026ttt3r}, we additionally report both per-sequence-scale and metric-scale depth on prepared KITTI validation clips truncated to 300, 400, and 500 frames.
Depth metrics are aggregated with valid-pixel weighting.

\paragraph{3D reconstruction.}
7-Scenes evaluation uses the prepared scene set with the frame budgets used in the paper (100, 200, 300, 500, and 1000).
For 7-Scenes, we subsample every 2nd frame (stride~2) for sequence lengths below 500 and use every frame (stride~1) for the length-1000 evaluation.
NRGBD evaluation uses the nine prepared scenes \texttt{breakfast\_room}, \texttt{complete\_kitchen}, \texttt{green\_room}, \texttt{grey\_white\_room}, \texttt{kitchen}, \texttt{morning\_apartment}, \texttt{staircase}, \texttt{thin\_geometry}, and \texttt{whiteroom}.
The reconstruction setup uses resolution $512$ as input with a $224$-pixel center crop for the ICP algorithm.
Metrics are Accuracy, Completeness, and Normal Consistency.

\paragraph{Short-horizon benchmarks.}
Following \cite{chen2026ttt3r}, short-horizon depth is evaluated on the prepared Sintel split (14 sequences), the Bonn dynamics set with standard length (110), and KITTI validation crops.
The short-horizon reconstruction table uses 7-Scenes at 100 and 200 frames, which therefore use stride~2.

\paragraph{Model selection and freeze point.}
A single fixed set of FILT3R hyperparameters is used across all tasks.
We do \emph{not} deploy different hyperparameters for pose and depth.
Unlike the fixed-$\beta$ baseline in Sec.~A, the final fixed-$R$ FILT3R model was not selected by a post hoc sweep on the long-horizon test tables.
No single hidden scalar objective was used for that final model; the goal was one shared operating point that remained competitive on both the pose and depth development sets.
Candidate fixed-$R$ configurations were screened once on the pose development set (TUM-200) and on a held-out short-horizon depth development set, after which one shared configuration was frozen and reused for every result reported in the main paper and this appendix.
The neighboring variant (higher-ceiling, later-opening gate) in Table~\ref{tab:hparam_stress} and the isolated long-horizon wins of smaller fixed-$\beta$ values were observed only in retained evaluation runs after this freeze point and did not influence model choice.
Table~\ref{tab:hparam_stress} therefore serves as a targeted local sensitivity check rather than a dense robustness study.

\begin{table*}[!tb]
\centering
\caption{\textbf{Targeted perturbation check around the main configuration.}
These runs come from retained evaluation logs and are intended as a local sensitivity check rather than a dense robustness study.
The nearby setting stays in the same qualitative regime across both pose and reconstruction.}
\label{tab:hparam_stress}
{\small (a) \textbf{TUM-800 pose}}
\vspace{1pt}

\resizebox{0.65\textwidth}{!}{%
\begin{tabular}{l|ccc}
\toprule
\textbf{Variant} & ATE$\downarrow$ & $\text{ATE}_{\text{orig}}\downarrow$ & RPE-r$\downarrow$ \\
\midrule
Ours & 0.057 & 0.107 & \textbf{0.362} \\
Gate variant ($q_{\max}{=}0.8,\tau_q{=}4.0$) & 0.057 & 0.107 & 0.364 \\
\bottomrule
\end{tabular}}

\vspace{4pt}
{\small (b) \textbf{7-Scenes reconstruction at length 300}}
\vspace{1pt}

\resizebox{0.65\textwidth}{!}{%
\begin{tabular}{l|ccc}
\toprule
\textbf{Variant} & Acc (mean)$\downarrow$ & Comp (mean)$\downarrow$ & NC (mean)$\uparrow$ \\
\midrule
Ours & 0.020 & 0.022 & 0.568 \\
Gate variant ($q_{\max}{=}0.8,\tau_q{=}4.0$) & 0.020 & 0.022 & 0.568 \\
\bottomrule
\end{tabular}}
\end{table*}

\paragraph{Per-sequence robustness statistics.}
The gains in the main tables are not driven by a single favorable sequence.
Table~\ref{tab:per_sequence} summarizes the per-sequence runs at the longest horizons for which outputs are available for all compared methods: TUM-800 for pose and Bonn-400 for metric depth.
FILT3R wins all eight TUM sequences on $\text{ATE}_{\text{orig}}$ and wins three of the five Bonn sequences on Abs Rel, while also achieving the best mean and worst-case depth errors.

 \section{Additional long-horizon depth results on KITTI}
\label{app:kitti_long}

\paragraph{Discussion.}
Under per-sequence scale alignment (Table~\ref{tab:depth_kitti_long}a), FILT3R consistently outperforms both CUT3R and TTT3R across all KITTI horizons.
Under metric-scale evaluation (Table~\ref{tab:depth_kitti_long}b), however, TTT3R holds a small edge over FILT3R---for example, Abs~Rel of 0.132 versus 0.134 at KITTI-500---while both remain substantially better than CUT3R.
We hypothesize that this is related to KITTI's predominantly forward-driving motion, which produces relatively steady temporal drift and thus a setting where TTT3R's heuristic gate may already be well-calibrated.
Even in this regime, the gap to FILT3R remains small (within 0.002 Abs~Rel at KITTI-500), while FILT3R remains stronger whenever per-sequence scale alignment is available.

\begin{table*}[!tb]
\centering
\caption{\textbf{Per-sequence robustness at long horizons.}
TUM-800 pose uses all eight prepared TUM-dynamics sequences.
Bonn-400 metric depth uses the five long Bonn sequences, which is the longest horizon for which per-sequence outputs are available for all compared methods.}
\label{tab:per_sequence}
{\small (a) \textbf{TUM-800 pose}}
\vspace{1pt}

\resizebox{0.88\textwidth}{!}{%
\begin{tabular}{l|ccc|ccc}
\toprule
\textbf{Method} & $\text{ATE}_{\text{orig}}$ mean$\pm$std$\downarrow$ & worst$\downarrow$ & wins$\uparrow$ & RPE-r mean$\pm$std$\downarrow$ & worst$\downarrow$ & wins$\uparrow$ \\
\midrule
CUT3R & 0.493 $\pm$ 0.318 & 1.137 & 0 & 0.486 $\pm$ 0.240 & 0.862 & 0 \\
TTT3R & 0.214 $\pm$ 0.097 & 0.427 & 0 & 0.387 $\pm$ 0.158 & 0.588 & 0 \\
TTT3R+Reset & 0.221 $\pm$ 0.060 & 0.290 & 0 & \textbf{0.338 $\pm$ 0.113} & \textbf{0.483} & \textbf{5} \\
\oursrow
\oursname & \textbf{0.107 $\pm$ 0.029} & \textbf{0.136} & \textbf{8} & 0.362 $\pm$ 0.148 & 0.581 & 3 \\
\bottomrule
\end{tabular}}

\vspace{4pt}
{\small (b) \textbf{Bonn-400 metric depth}}
\vspace{1pt}

\resizebox{0.88\textwidth}{!}{%
\begin{tabular}{l|ccc|ccc}
\toprule
\textbf{Method} & Abs Rel mean$\pm$std$\downarrow$ & worst$\downarrow$ & wins$\uparrow$ & $\delta{<}1.25$ mean$\pm$std$\uparrow$ & worst$\uparrow$ & wins$\uparrow$ \\
\midrule
CUT3R & 0.111 $\pm$ 0.049 & 0.190 & 0 & 88.3 $\pm$ 12.7 & 64.1 & 1 \\
TTT3R & 0.108 $\pm$ 0.053 & 0.196 & 0 & 90.3 $\pm$ 10.2 & 71.6 & 1 \\
TTT3R+Reset & 0.097 $\pm$ 0.040 & 0.148 & 2 & 92.4 $\pm$ 6.5 & 82.1 & 1 \\
\oursrow
\oursname & \textbf{0.092 $\pm$ 0.035} & \textbf{0.139} & \textbf{3} & \textbf{93.8 $\pm$ 4.4} & \textbf{87.5} & \textbf{2} \\
\bottomrule
\end{tabular}}
\end{table*}

\begin{table*}[!tb]
\centering
\caption{\textbf{Long-horizon video depth on KITTI.}
(a) Per-sequence scale alignment and (b) metric-scale evaluation are both reported on prepared KITTI validation clips truncated to 300, 400, and 500 frames.}
\label{tab:depth_kitti_long}

{\small (a) \textbf{Per-sequence scale} (KITTI-300 / KITTI-400 / KITTI-500)}
\vspace{1pt}

\resizebox{\textwidth}{!}{%
\begin{tabular}{l | ccc | ccc | ccc}
\toprule
\multirow{2}{*}{\textbf{Method}} &
\multicolumn{3}{c|}{\textbf{KITTI-300}} &
\multicolumn{3}{c|}{\textbf{KITTI-400}} &
\multicolumn{3}{c}{\textbf{KITTI-500}}\\
\cmidrule(lr){2-4}\cmidrule(lr){5-7}\cmidrule(lr){8-10}
& Abs Rel$\downarrow$ & $\delta{<}1.25\uparrow$ & log RMSE$\downarrow$ & Abs Rel$\downarrow$ & $\delta{<}1.25\uparrow$ & log RMSE$\downarrow$ & Abs Rel$\downarrow$ & $\delta{<}1.25\uparrow$ & log RMSE$\downarrow$ \\
\midrule
CUT3R       & 0.122 & 87.5 & 0.165 & 0.127 & 86.9 & 0.170 & 0.130 & 86.5 & 0.175 \\
TTT3R       & 0.111 & 90.3 & 0.152 & 0.115 & 89.7 & 0.159 & 0.118 & 89.2 & 0.165 \\
\oursrow
\oursname   & \textbf{0.108} & \textbf{91.1} & \textbf{0.149} & \textbf{0.112} & \textbf{90.2} & \textbf{0.156} & \textbf{0.116} & \textbf{89.6} & \textbf{0.162} \\
\bottomrule
\end{tabular}}

\vspace{4pt}
{\small (b) \textbf{Metric scale} (KITTI-300 / KITTI-400 / KITTI-500)}
\vspace{1pt}

\resizebox{\textwidth}{!}{%
\begin{tabular}{l | ccc | ccc | ccc}
\toprule
\multirow{2}{*}{\textbf{Method}} &
\multicolumn{3}{c|}{\textbf{KITTI-300}} &
\multicolumn{3}{c|}{\textbf{KITTI-400}} &
\multicolumn{3}{c}{\textbf{KITTI-500}}\\
\cmidrule(lr){2-4}\cmidrule(lr){5-7}\cmidrule(lr){8-10}
& Abs Rel$\downarrow$ & $\delta{<}1.25\uparrow$ & log RMSE$\downarrow$ & Abs Rel$\downarrow$ & $\delta{<}1.25\uparrow$ & log RMSE$\downarrow$ & Abs Rel$\downarrow$ & $\delta{<}1.25\uparrow$ & log RMSE$\downarrow$ \\
\midrule
CUT3R       & 0.133 & 83.8 & 0.179 & 0.143 & 82.0 & 0.187 & 0.151 & 80.3 & 0.194 \\
TTT3R       & \textbf{0.117} & \textbf{88.6} & \textbf{0.161} & \textbf{0.125} & \textbf{87.3} & \textbf{0.169} & \textbf{0.132} & \textbf{86.5} & \textbf{0.176} \\
\oursrow
\oursname   & 0.119 & \textbf{88.6} & \textbf{0.161} & 0.127 & \textbf{87.3} & \textbf{0.169} & 0.134 & 86.2 & 0.177 \\
\bottomrule
\end{tabular}}
\end{table*}

\section{Complete short-horizon tables}
\label{app:depth_short_full}

Tables~\ref{tab:depth_short_full_app} and~\ref{tab:recon_short} present the full short-horizon comparisons for depth and reconstruction, respectively.
These tables include optimization-based and full-attention baselines in addition to the streaming methods compared in the main paper.

On short-horizon depth (Table~\ref{tab:depth_short_full_app}), FILT3R is the strongest streaming method on Sintel and Bonn under metric-scale evaluation and remains competitive on KITTI; under per-sequence scale alignment it is competitive across all three datasets.

\paragraph{Scope of the empirical claim.}
The short-horizon gains are smaller than the long-horizon ones, so our central empirical claim is long-horizon stabilization rather than universal dominance across all regimes.

\begin{table*}[!tb]
\centering
\caption{\textbf{Video depth estimation on short sequences (full comparison).}
Both per-sequence scale and metric-scale results are shown. Rows marked with~\prior\ are reproduced from prior work.}
\label{tab:depth_short_full_app}
\resizebox{\linewidth}{!}{%
\begin{tabular}{llc|cc|cc|cc}
\toprule
\textbf{Alignment} & \textbf{Method} & \textbf{Type} &
\multicolumn{2}{c|}{\textbf{Sintel}} &
\multicolumn{2}{c|}{\textbf{Bonn}} &
\multicolumn{2}{c}{\textbf{KITTI}}\\
\cmidrule(lr){4-5}\cmidrule(lr){6-7}\cmidrule(lr){8-9}
 &  &  & Abs Rel$\downarrow$ & $\delta{<}1.25\uparrow$ & Abs Rel$\downarrow$ & $\delta{<}1.25\uparrow$ & Abs Rel$\downarrow$ & $\delta{<}1.25\uparrow$\\
\midrule
\multirow{14}{*}{Per-seq.\ scale}
& DUSt3R-GA\prior~\cite{wang2024dust3r}    & Optim  & 0.656 & 45.2 & 0.155 & 83.3 & 0.144 & 81.3\\
& MASt3R-GA\prior~\cite{leroy2024mast3r}    & Optim  & 0.641 & 43.9 & 0.252 & 70.1 & 0.183 & 74.5\\
& MonST3R-GA\prior~\cite{zhang2025monst3r}   & Optim  & 0.378 & 55.8 & 0.067 & 96.3 & 0.168 & 74.4\\
& Easi3R\prior~\cite{chen2025easi3r}       & FA     & 0.377 & 55.9 & 0.059 & 97.0 & 0.102 & 91.2\\
& VGGT\prior         & FA     & 0.287 & 66.1 & 0.055 & 97.1 & 0.070 & 96.5\\
& DA3~\cite{lin2025depthanything3} & FA & 0.278 & 70.3 & 0.052 & 97.2 & 0.045 & 98.0 \\
\cmidrule(lr){2-9}
& Spann3R\prior~\cite{wang2024spann3r}      & Stream & 0.622 & 42.6 & 0.144 & 81.3 & 0.198 & 73.7\\
& Point3R\prior      & Stream & 0.452 & 48.9 & 0.060 & 96.0 & 0.136 & 84.2\\
& CUT3R\prior        & Stream & 0.421 & 47.9 & 0.078 & 93.7 & 0.118 & 88.1\\
& TTT3R\prior        & Stream & 0.405 & 48.9 & 0.069 & 95.4 & 0.114 & 90.4\\
& STREAM3R$^{\alpha}$\prior~\cite{lan2025stream3r} & Stream & 0.478 & 51.1 & 0.075 & 94.1 & 0.116 & 89.6\\
& StreamVGGT\prior~\cite{zhuo2025streamvggt}   & Stream & \textbf{0.323} & \textbf{65.7} & \textbf{0.059} & \textbf{97.2} & 0.173 & 72.1\\
& MUT3R\prior        & Stream & 0.451 & 48.6 & 0.070 & 96.2 & 0.116 & 88.3\\
& FILT3R (ours)      & Stream & \underline{0.407} & \underline{54.5} & \underline{0.061} & \underline{97.0} & \textbf{0.110} & \textbf{91.0}\\
\midrule
\multirow{8}{*}{Metric scale}
& MASt3R-GA\prior    & Optim  & 1.022 & 14.3 & 0.272 & 70.6 & 0.467 & 15.2\\
& DA3~\cite{lin2025depthanything3} & FA & 0.529 & 6.2  & 0.147 & 88.8 & 0.237 & 28.5 \\
\cmidrule(lr){2-9}
& CUT3R\prior        & Stream & 1.029 & 23.8 & 0.103 & 88.5 & 0.122 & 85.5\\
& Point3R\prior      & Stream & 0.777 & 17.1 & 0.137 & 94.7 & 0.191 & 73.8\\
& TTT3R              & Stream & 0.977 & 24.5 & 0.090 & 94.2 & 0.110 & 89.1\\
& STREAM3R$^{\alpha}$\prior & Stream & 1.041 & 21.0 & 0.084 & 94.4 & 0.234 & 57.6\\
& MUT3R\prior        & Stream & 0.820 & 25.2 & 0.086 & 96.0 & 0.125 & 85.8\\
& FILT3R (ours)      & Stream & \textbf{0.772} & \textbf{27.3} & \textbf{0.067} & \textbf{96.7} & \underline{0.115} & 88.8 \\
\bottomrule
\end{tabular}}
\end{table*}

\begin{table*}[!tb]
\centering
\caption{\textbf{3D reconstruction on 7-Scenes at short sequences (100 and 200 frames).}
Even near the training horizon, FILT3R is already competitive with or slightly better than prior streaming updates, though the gap is considerably smaller than in the long-horizon settings reported above.}
\label{tab:recon_short}
\resizebox{\textwidth}{!}{%
\begin{tabular}{l | cccccc | cccccc}
\toprule
\multirow{3}{*}{\textbf{Method}} & \multicolumn{6}{c|}{\textbf{Length 100}} & \multicolumn{6}{c}{\textbf{Length 200}} \\
\cmidrule(lr){2-7} \cmidrule(lr){8-13}
 & \multicolumn{2}{c}{Acc$\downarrow$} & \multicolumn{2}{c}{Comp$\downarrow$} & \multicolumn{2}{c|}{NC$\uparrow$} & \multicolumn{2}{c}{Acc$\downarrow$} & \multicolumn{2}{c}{Comp$\downarrow$} & \multicolumn{2}{c}{NC$\uparrow$} \\
\cmidrule(lr){2-3} \cmidrule(lr){4-5} \cmidrule(lr){6-7} \cmidrule(lr){8-9} \cmidrule(lr){10-11} \cmidrule(lr){12-13}
 & Mean & Med. & Mean & Med. & Mean & Med. & Mean & Med. & Mean & Med. & Mean & Med. \\
\midrule
CUT3R         & 0.033 & 0.018 & 0.022 & 0.005 & 0.592 & 0.644 & 0.089 & 0.053 & 0.048 & 0.017 & 0.566 & 0.600 \\
TTT3R         & 0.020 & 0.008 & \textbf{0.016} & \textbf{0.004} & \textbf{0.596} & \textbf{0.649} & 0.028 & 0.015 & 0.023 & 0.005 & 0.582 & 0.627 \\
\oursrow
\oursname     & \textbf{0.017} & \textbf{0.006} & 0.017 & \textbf{0.004} & 0.595 & 0.648 & \textbf{0.019} & \textbf{0.008} & \textbf{0.022} & \textbf{0.004} & \textbf{0.584} & \textbf{0.629} \\
\bottomrule
\end{tabular}}
\end{table*}

\section{Hyperparameter details}
\label{app:hyperparams}
Table~\ref{tab:hyperparams} lists the complete set of hyperparameters used by FILT3R.
These values are fixed across all tasks and benchmarks reported in both the main paper and this appendix.
As discussed in Sec.~\ref{sec:gain_floor}, the positive $q_{\min}$ together with gain clamping implies that the deployed filter has a finite stable-regime gain floor rather than an asymptotically vanishing gain.

\begin{table}[!ht]
\centering
\caption{\textbf{FILT3R hyperparameters.}}
\label{tab:hyperparams}
\resizebox{0.6\linewidth}{!}{%
\begin{tabular}{lll}
\toprule
\textbf{Symbol} & \textbf{Description} & \textbf{Value} \\
\midrule
$p_0$ & Initial per-token variance & 1.5 \\
$k_{\min}, k_{\max}$ & Kalman-style gain clamp bounds & 0.01, 0.99 \\
$q_{\min}$, $q_{\max}$ & Process noise bounds & 0.02, 0.5 \\
$\alpha_q$, $\tau_q$ & Process-noise sigmoid parameters & 20, 3.0 \\
$r$ & Fixed measurement-noise scalar & 1.0 \\
$\lambda_\Delta$ & Drift EMA rate & 0.05 \\
$\Delta_{\mathrm{floor}}$ & Drift EMA floor & 0.01 \\
$\varepsilon$ & Numerical stability constant & $10^{-6}$ \\
\bottomrule
\end{tabular}}
\end{table}

\section{Additional qualitative evidence of long-horizon stability}
\label{app:qualitative_long}

\begin{figure}[!b]
\centering
\includegraphics[width=0.8\linewidth]{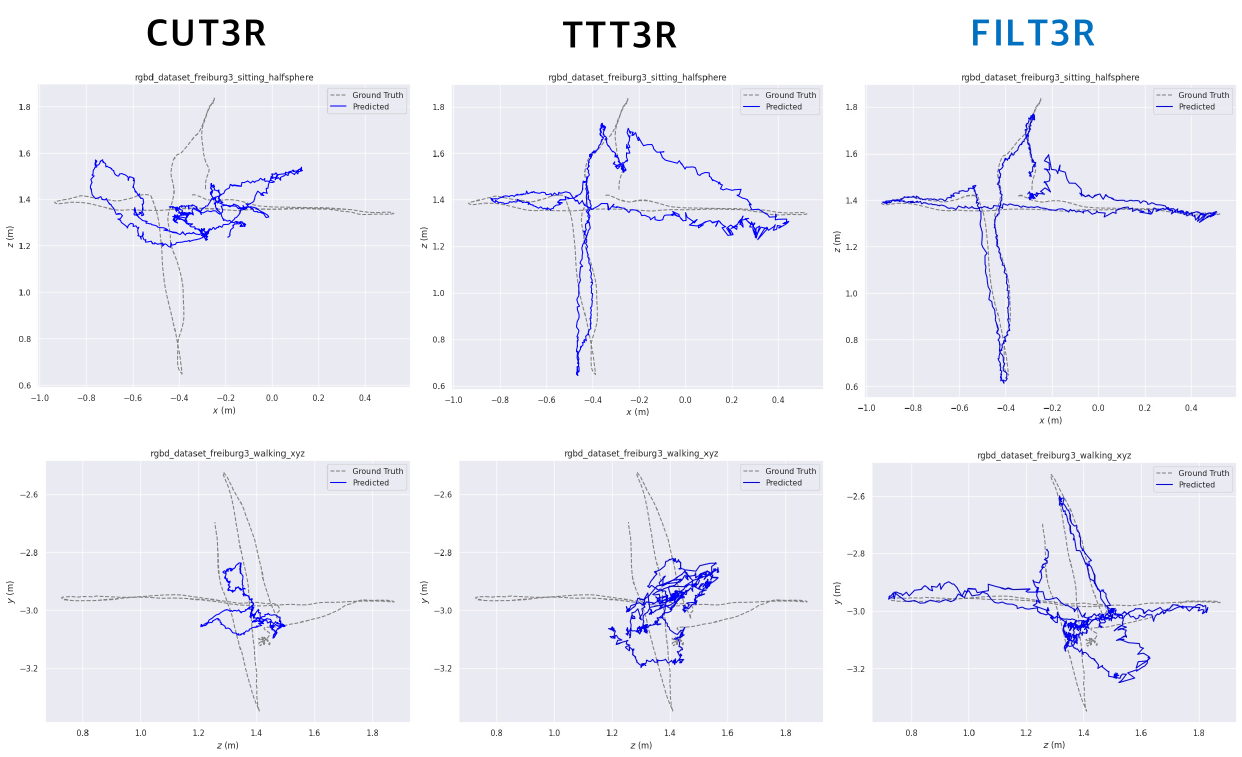}
\caption{\textbf{TUM-RGBD trajectory visualizations.}
FILT3R stays closer to the ground-truth trajectory over long horizons than both the overwrite and heuristic-gating baselines, particularly after the rollout extends well beyond the training horizon.}
\label{fig:pose_tum_qualitative}
\end{figure}

\begin{figure}[!tb]
\includegraphics[width=0.98\linewidth]{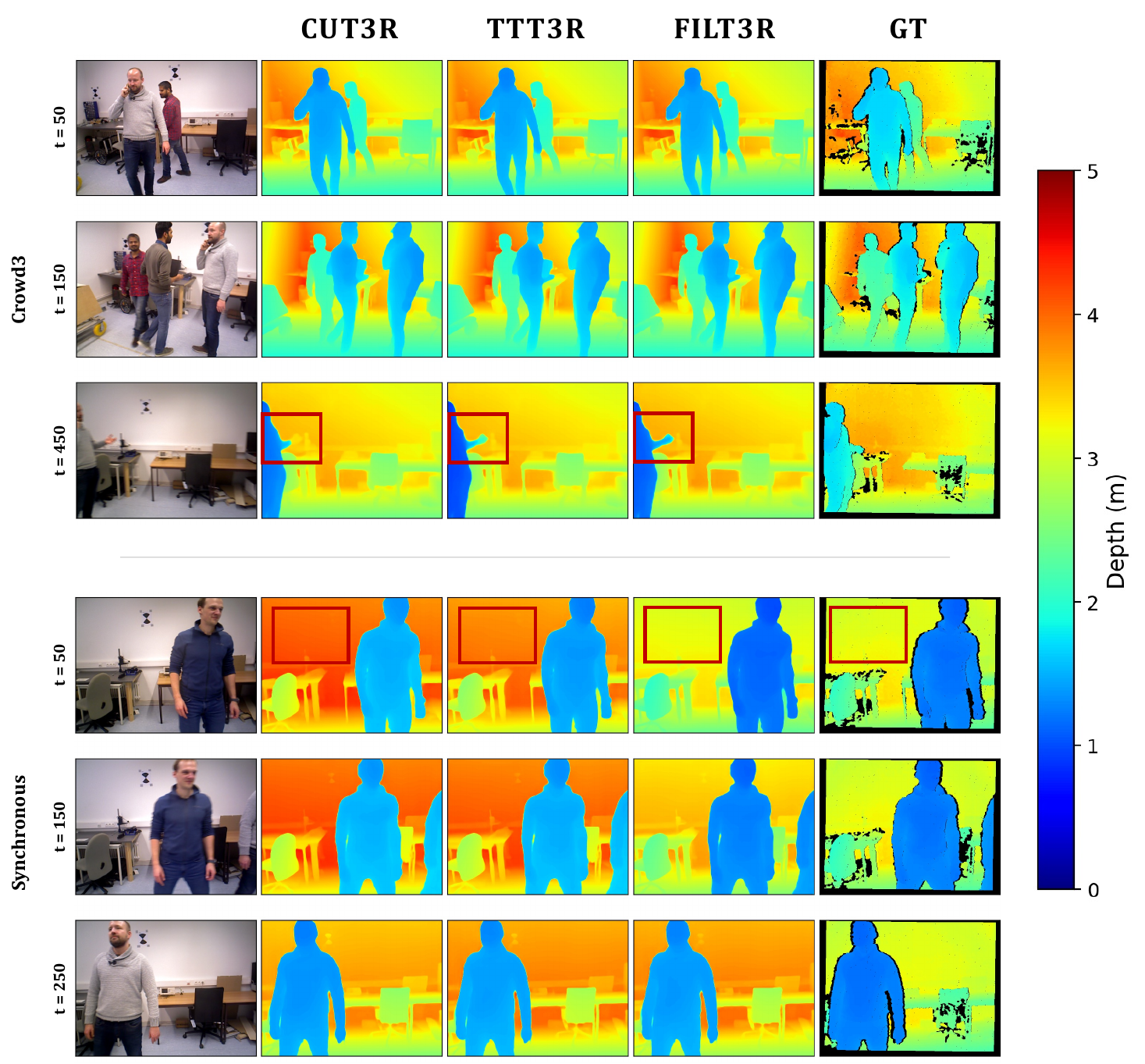}
\caption{\textbf{Bonn depth prediction visualizations.}
FILT3R stays closer to the ground-truth depth over long horizons than both the overwrite and heuristic-gating baselines, producing more metric-aligned depth predictions with sharper detail.}
\label{fig:pose_bonn_qualitative}
\end{figure}

\FloatBarrier
\clearpage

\section{FILT3R on fast-weight memory (tttLRM): integration and detailed results}
\label{app:tttlrm}

This appendix expands the fast-weight generality experiment summarized in \cref{sec:generality} and \cref{tab:tttlrm_len}.
We detail the integration mechanism, the full length- and chunk-axis benchmarks on DL3DV-140, and the induced gain dynamics.
All runs use the public tttLRM autoregressive checkpoint with \emph{frozen} learned parameters; FILT3R changes only the fast-weight update rule and adds no trainable parameters.

\subsection{Integration into the LaCT/Muon fast-weight update}
\label{app:tttlrm_integration}
tttLRM~\cite{wang2026tttlrm} represents per-scene memory as a set of \emph{fast weights} $\mathbf{W}=\{\mathbf{W}_0,\mathbf{W}_1,\mathbf{W}_2\}$ that are refreshed as the autoregressive scan ingests views, using a LaCT/Muon test-time transition.
We map FILT3R onto this memory by treating the committed fast weights $\mathbf{W}_{t-1}$ as the latent state and the post-transition candidate $\tilde{\mathbf{W}}_t$ as a noisy measurement, then applying the same scalar Kalman-style recursion (\cref{eq:p_pred,eq:kalman_gain,eq:state_update_main,eq:p_update}) per weight group:
\begin{align}
d_t &= \frac{\mathrm{RMS}(\tilde{\mathbf{W}}_t-\mathbf{W}_{t-1})}{\mathrm{RMS}(\mathbf{W}_{t-1})}, &
\hat d_t &= \beta_d\,\hat d_{t-1}+(1-\beta_d)\,d_t, \\
\mathbf{q}_t &= q_{\mathrm{base}}\cdot\mathrm{clip}\!\big(d_t/\hat d_t,\ q_{\min},q_{\max}\big), &
\mathbf{p}^-_t &= \mathrm{clip}\!\big(\mathbf{p}_{t-1}+\mathbf{q}_t,\ 10^{-6},\,10^{2}\big), \\
\mathbf{k}_t &= \mathrm{clip}\!\big(\mathbf{p}^-_t/(\mathbf{p}^-_t+r),\ k_{\min},k_{\max}\big), &
\mathbf{W}_t &= \mathbf{W}_{t-1}+\mathbf{k}_t\odot(\tilde{\mathbf{W}}_t-\mathbf{W}_{t-1}),
\end{align}
with the update $\mathbf{p}_t=(1-\mathbf{k}_t)\,\mathbf{p}^-_t$.
Here the relative RMS change $d_t$ plays the role of the candidate temporal drift in the main method (\cref{eq:drift}).
Three implementation points are worth noting.
\begin{itemize}
\item \textbf{Granularity.} In \emph{scalar} mode we keep one $(\mathbf{p},\mathbf{k})$ per attention head per weight matrix; in \emph{row-wise} mode we keep a per-input-row $(\mathbf{p},\mathbf{k})$, giving per-row credit assignment. Row-wise is the FILT3R-FW variant reported in \cref{tab:tttlrm_len}. In the default scale the scalar variant's gain floor binds and it reduces \emph{exactly} to a constant-gain EMA, which is the ``Fixed-gain EMA'' control throughout.
\item \textbf{Norm convention.} tttLRM keeps a column-norm convention on the fast weights; after the filtered mix we renormalize each column to the candidate's norm, so filtering changes the direction and magnitude of the \emph{applied} update while preserving the architecture's norm invariant.
\item \textbf{Warm-up.} The first $w$ chunks are committed unfiltered to seed the memory, after which filtering engages; the best length-128 result uses $w{=}0$ (no warm-up).
\end{itemize}
All FILT3R operations run under \texttt{no\_grad} after the existing transition, adding negligible compute and ${\approx}\,300$\,MiB of state.

\subsection{Length scaling at the standard update step}
\Cref{tab:tttlrm_len} reports the DL3DV-140 benchmark at the standard LaCT chunk size, across sequence lengths $16,64,128$.
The within-length PSNR gain over the unmodified autoregressive update grows monotonically with length, from $+0.05$\,dB at length 16 to $+1.08$\,dB at length 128, with SSIM/LPIPS moving in the same direction and 127--140 of the 140 scenes improving.
Absolute PSNR is not directly comparable across lengths because the target-view sampling protocol differs; the within-length $\Delta$ is the relevant signal.

\begin{table}[!tb]
  \centering
  \caption{\textbf{tttLRM length scaling at the standard LaCT step (DL3DV-140).}
  \emph{Original-AR} is the unmodified autoregressive update; \emph{Fixed-gain EMA} is the scalar floor-binding control; \emph{+\,FILT3R-FW} is the row-wise adaptive variant. The \emph{wins} column counts scenes (out of 140) beating Original-AR. The gain $\Delta$ grows with sequence length.}
  \label{tab:tttlrm_len}
  \resizebox{\linewidth}{!}{%
  \begin{tabular}{c|l|ccccc}
    \toprule
    Len. & Method & PSNR$\uparrow$ & $\Delta$ & SSIM$\uparrow$ & LPIPS$\downarrow$ & wins \\
    \midrule
    \multirow{3}{*}{16}
      & \textit{Original-AR} & 23.0966 & -- & 0.7638 & 0.2681 & -- \\
      & Fixed-gain EMA & 23.1423 & +0.0457 & 0.7648 & 0.2677 & 127/140 \\
      \oursrow & +\,FILT3R-FW & \textbf{23.1457} & \textbf{+0.0491} & \textbf{0.7648} & \textbf{0.2677} & \textbf{127/140} \\
    \midrule
    \multirow{3}{*}{64}
      & \textit{Original-AR} & 24.8116 & -- & 0.8140 & 0.2253 & -- \\
      & Fixed-gain EMA & 24.9444 & +0.1328 & 0.8173 & 0.2230 & -- \\
      \oursrow & +\,FILT3R-FW & \textbf{24.9647} & \textbf{+0.1531} & \textbf{0.8192} & \textbf{0.2220} & -- \\
    \midrule
    \multirow{3}{*}{128}
      & \textit{Original-AR} & 23.1710 & -- & 0.7794 & 0.2676 & -- \\
      & Fixed-gain EMA & 23.5922 & +0.4212 & 0.7936 & 0.2527 & \textbf{140/140} \\
      \oursrow & +\,FILT3R-FW & \textbf{24.2536} & \textbf{+1.0826} & \textbf{0.8138} & \textbf{0.2314} & \textbf{140/140} \\
    \bottomrule
  \end{tabular}}
\end{table}

\subsection{Chunk-axis stress test: catastrophic forgetting and rescue}
The dominant FILT3R signal appears not by adding frames but by \emph{shrinking the LaCT update step}, i.e.\ forcing more, smaller fast-weight updates per scene (controlled by \texttt{miniupdate\_views}).
\Cref{tab:tttlrm_chunks} sweeps this chunk axis at two sequence lengths.
As the number of chunks grows from 4 to 16, the unfiltered fast weights suffer a ${\approx}\,{-}5.8$\,dB collapse---the canonical catastrophic-forgetting failure in which repeated small over-writes corrupt the memory---and the collapse magnitude is essentially independent of sequence length, identifying it as a property of the recurrent fast-weight dynamics rather than a frame-count artifact.
FILT3R recovers a large fraction of this loss: up to $+1.618$\,dB (length 16) and $+1.542$\,dB (length 20) at the finest step, on \emph{all} 140 scenes.
Row-wise FILT3R increasingly dominates the scalar/EMA control as chunks proliferate (tied at 4 chunks; more than $2\times$ the gain at 16 chunks), consistent with the per-row gain behavior analyzed next.

\begin{table}[!tb]
  \centering
  \caption{\textbf{Chunk-axis stress test on DL3DV-140.}
  Shrinking the LaCT update step (more chunks) collapses the unfiltered baseline by ${\approx}\,5.8$\,dB; FILT3R-FW (row-wise) recovers up to $+1.6$\,dB on all 140 scenes. Baseline is absolute PSNR; the two FILT3R columns are $\Delta$PSNR over the baseline.}
  \label{tab:tttlrm_chunks}
  \resizebox{\linewidth}{!}{%
  \begin{tabular}{c|c|c|cc|c}
    \toprule
    Len. & chunks (filtered) & \textit{Original-AR} & Fixed-gain EMA $\Delta$ & +\,FILT3R-FW $\Delta$ & wins \\
    \midrule
    \multirow{3}{*}{16}
      & 4 (3) & 23.097 & +0.046 & +0.049 & 127/140 \\
      & 8 (7) & 21.290 & +0.480 & \textbf{+0.764} & \textbf{140/140} \\
      & 16 (15) & 17.255 & +0.702 & \textbf{+1.618} & \textbf{140/140} \\
    \midrule
    \multirow{3}{*}{20}
      & 5 (4) & 23.597 & +0.050 & +0.052 & 130/140 \\
      & 10 (9) & 21.588 & +0.530 & \textbf{+0.923} & \textbf{140/140} \\
      & 20 (19) & 17.724 & +0.605 & \textbf{+1.542} & \textbf{140/140} \\
    \bottomrule
  \end{tabular}}
\end{table}

\subsection{Induced gain dynamics}
\Cref{tab:tttlrm_gain} logs the row-wise gain over the 15 filtered chunks at the finest update step.
The mean gain anneals from $0.90$ toward $0.65$ as evidence accumulates---mirroring the $\mathcal{O}(1/t)$-style contraction analyzed in Appendix~\ref{app:memory_horizon}, now realized on fast weights---while the per-row spread grows from $0$ to ${\approx}\,0.046$ (late-chunk range $[0.51,0.79]$).
Some rows are strongly retained ($k{\approx}0.5$, a $50/50$ mix) while others keep updating ($k{\approx}0.79$): the adaptive process noise turns the update into a per-row low-pass filter on the noisy update direction, which is the mechanism behind the rescue in \cref{tab:tttlrm_chunks}.

\begin{table}[!tb]
  \centering
  \caption{\textbf{Row-wise gain trajectory at the finest update step} (15 filtered chunks, logged on a DL3DV sample). The mean gain anneals downward and the per-row spread grows as consistent evidence accumulates.}
  \label{tab:tttlrm_gain}
  \resizebox{0.7\linewidth}{!}{%
  \begin{tabular}{c|cccc}
    \toprule
    chunk & $k_{\mathrm{mean}}$ & $k_{\min}$ & $k_{\max}$ & per-row $k_{\mathrm{std}}$ \\
    \midrule
    1 & 0.900 & 0.900 & 0.900 & 0.000 \\
    2 & 0.827 & 0.731 & 0.830 & 0.010 \\
    3 & 0.798 & 0.674 & 0.829 & 0.025 \\
    4 & 0.769 & 0.652 & 0.826 & 0.027 \\
    5 & 0.738 & 0.590 & 0.826 & 0.040 \\
    6 & 0.712 & 0.551 & 0.826 & 0.048 \\
    7 & 0.694 & 0.542 & 0.820 & 0.045 \\
    8 & 0.682 & 0.536 & 0.808 & 0.044 \\
    $\vdots$ & $\vdots$ & $\vdots$ & $\vdots$ & $\vdots$ \\
    15 & 0.645 & 0.515 & 0.793 & 0.046 \\
    \bottomrule
  \end{tabular}}
\end{table}

\subsection{Hyperparameters}
\Cref{tab:tttlrm_hparams} lists the fast-weight FILT3R configuration.
Because the fast-weight state has a different scale and update cadence than the CUT3R latent, the noise constants differ from the main-paper values in \cref{tab:hyperparams}; within the tttLRM experiments a single configuration is used for all lengths and chunk settings.
The scalar ``Fixed-gain EMA'' control uses the same recursion in scalar mode with a higher $r$ so that the gain floor binds and the gain is effectively constant.

\begin{table}[!ht]
  \centering
  \caption{\textbf{tttLRM FILT3R hyperparameters.} The drift clip $(q_{\min},q_{\max})$ bounds the normalized ratio $d_t/\hat d_t$ before scaling by $q_{\mathrm{base}}$.}
  \label{tab:tttlrm_hparams}
  \resizebox{0.9\linewidth}{!}{%
  \begin{tabular}{l|cc}
    \toprule
    \textbf{Setting} & \textbf{FILT3R-FW (row-wise)} & \textbf{Fixed-gain EMA (scalar)} \\
    \midrule
    granularity & row-wise & scalar \\
    $q_{\mathrm{base}}$ & 0.01 & --- (gain floor binds) \\
    drift clip $(q_{\min},q_{\max})$ & $(0.25,\ 4.0)$ & --- \\
    measurement noise $r$ & 0.01 & 1.0 \\
    initial variance $p_0$ & 0.05 & (floor-binding) \\
    gain clamp $(k_{\min},k_{\max})$ & $(0.5,\ 0.995)$ & $(0.88,\ 0.995)$ \\
    drift EMA rate $\beta_d$ & 0.90 & 0.90 \\
    warm-up chunks $w$ & 1 (0 at length 128) & 1 \\
    \bottomrule
  \end{tabular}}
\end{table}

\FloatBarrier

\clearpage
\section{Generality to pointer memory (Point3R)}
\label{app:memory_designs}
This section provides the full memory-design comparison summarized in \cref{sec:generality}, evaluated on 7-Scenes-300.
Beyond the compact CUT3R latent and the tttLRM fast-weight memory (Appendix~\ref{app:tttlrm}), we test FILT3R on an \emph{explicit spatial pointer} memory and position it against an \emph{attention/KV-cache} memory.
Point3R~\cite{wu2025point3r} maintains an explicit spatial pointer memory; replacing only its pointer-fusion rule with a filtered update---leaving matching, grouping, and the decoder unchanged---improves Accuracy and Completeness on 7-Scenes-300 (\cref{tab:sevenscenes}), confirming that FILT3R generalizes across \emph{fast-weight}, \emph{explicit-pointer}, and \emph{compact-latent} memories.
\Cref{tab:sevenscenes} also compares against InfiniteVGGT~\cite{infinitevggt}, an attention/KV-cache memory: the compact CUT3R+FILT3R model attains the best Acc and Comp and is within $0.002$ NC, while running at constant memory and higher throughput than the attention baseline (FILT3R: $3.5$\,GB, $25$\,fps, cf.\ \cref{tab:efficiency}; InfiniteVGGT: $18.5$\,GB, $6.5$\,fps).

\begin{table}[!htb]
  \centering
  \caption{\textbf{7-Scenes-300 comparison across memory designs.}
  Point3R validates the pointer-memory fusion variant of FILT3R, while InfiniteVGGT provides an attention/KV-cache memory comparison. The main CUT3R+FILT3R model retains a compact latent memory and achieves the best Acc/Comp, with NC within $0.002$ of the heavier attention-memory baseline.}
  \label{tab:sevenscenes}
  \resizebox{\linewidth}{!}{%
  \begin{tabular}{l|l|ccc}
    \toprule
    Method & Memory form & Acc$\downarrow$ & Comp$\downarrow$ & NC$\uparrow$ \\
    \midrule
    Point3R & explicit pointer memory & 0.049 & 0.029 & 0.562 \\
    \oursrow Point3R + FILT3R & filtered pointer fusion & 0.042 & 0.024 & 0.563 \\
    InfiniteVGGT & attention/KV-cache memory & 0.040 & 0.025 & \textbf{0.570} \\
    \oursrow \oursname{} (main) & compact latent memory & \textbf{0.020} & \textbf{0.022} & 0.568 \\
    \bottomrule
  \end{tabular}}
\end{table}

\FloatBarrier

\end{document}